\documentclass{article} % For LaTeX2e
\usepackage{iclr2026_conference,times}

\usepackage{float}
\usepackage{wrapfig}
\usepackage{amsmath,amsfonts,bm}

\def\eqref#1{equation~\ref{#1}}
\def\1{\bm{1}}

\DeclareMathAlphabet{\mathsfit}{\encodingdefault}{\sfdefault}{m}{sl}
\SetMathAlphabet{\mathsfit}{bold}{\encodingdefault}{\sfdefault}{bx}{n}

\usepackage[colorlinks=true,linkcolor=red!60!black,citecolor=green!45!black,urlcolor=blue!60!black]{hyperref}
\usepackage{url}
\usepackage{graphicx}
\usepackage{booktabs}
\usepackage{longtable}
\usepackage{amsmath,amssymb,amsthm}
\usepackage{xcolor}
\newtheorem{proposition}{Proposition}

\newcommand{\todo}[1]{}
\newcommand{\pending}[1]{}

\newcommand{\rpred}{r_{\mathrm{pred}}}

\title{The Dark Room in the Reward Channel: Dense Prediction Rewards Collapse
GRPO-Trained LLM Agents, and the Channel, Not the Content, Decides What Works}

\author{Yu Wang\thanks{Independent researcher. Correspondence: \texttt{wangyu099508@gmail.com}.
ORCID: \href{https://orcid.org/0009-0009-6231-2484}{0009-0009-6231-2484}.
Preprint: \href{https://arxiv.org/abs/2607.21273}{arXiv:2607.21273}; an earlier version is
deposited as \href{https://doi.org/10.5281/zenodo.21505228}{doi:10.5281/zenodo.21505228}.}}

\newcommand{\fix}[1]{}
\iclrfinalcopy % arXiv preprint: non-anonymous

\begin{document}

\maketitle
% arXiv header: override the ICLR style banner (set after maketitle, which resets it)
\lhead{Preprint. Under review.}
\thispagestyle{fancy}

\begin{abstract}
% Abstract: problem -> related work/gap -> approach -> numbered results,
% 250-300 words per venue norm. Every retained qualifier is audit-hardened
% (25 external rounds): the mitigation parenthetical, the floor-bound
% exemption, seed-pairing labels, and the regime clause are load-bearing --
% do not drop or widen when editing. Supporting claims cut for length
% (slow-anneal arm, compute control, WebShop split, six-run hedge) live in
% S1/S4 with their full scoping.
Sparse success rewards make long-horizon LLM agents slow to train; the
common remedy is dense per-step supervision: reward the policy for predicting
its next observation, which looks provably safe under
potential-based shaping. Published prediction-reward and
auxiliary-loss variants report successes and instabilities; no
controlled account establishes when the reward channel is dangerous or where
auxiliary gains come from. We give that account: 74 preregistered
arms dissect one fixed prediction signal under group-normalized RL (GRPO)
across ALFWorld, WebShop, a synthetic POMDP, and Qwen3-1.7B/4B/8B,
varying only the delivery mechanism.
(1)~Every run sustaining this difference-form reward under untouched std
normalization (no filtering, dynamic-sampling, or decoupling mitigations)
collapses: eleven runs across scales, coefficients, group sizes, and
groupings (the floor-bound synthetic environment stalls instead); ALFWorld runs end in a degenerate absorbing state (prediction accuracy
${\to}1.0$, success ${\to}0$): the optimizer builds the ``dark room''. One
line of algebra explains it: inside all-fail groups, z-scoring
cancels the shaping coefficient; removing only std normalization restores
baseline parity.
(2)~A signal's danger is set by its within-group \emph{variance
trajectory}, with hackability as a second axis; it retrodicts every
reward-channel collapse and survives preregistered prospective tests.
(3)~The same signal as an auxiliary teacher-forced loss does no harm on
ALFWorld at 4B, but the gain is not the signal's: content-free placebos as a
class match or beat gold at both matched seeds (seed 0: placebo mean 78.8 vs.\ 68.6;
seed 42: single endpoint 67.9 vs.\ 57.9); the auxiliary
update is the regularizer.
(4)~At 8B the recipe turns bistable: gold full-weight locks two of three seeds;
every content-free or reduced-weight arm stays healthy.
No ALFWorld or WebShop reward-channel variant measurably
beats its matched-normalization baseline and no gold signal measurably
outperforms its content-free placebo: the delivery channel, not the content,
decides; which channel is safe is regime-dependent.

\end{abstract}

\section{Introduction}
\label{sec:intro}

\begin{figure}[t]
  \centering
  \includegraphics[width=\linewidth]{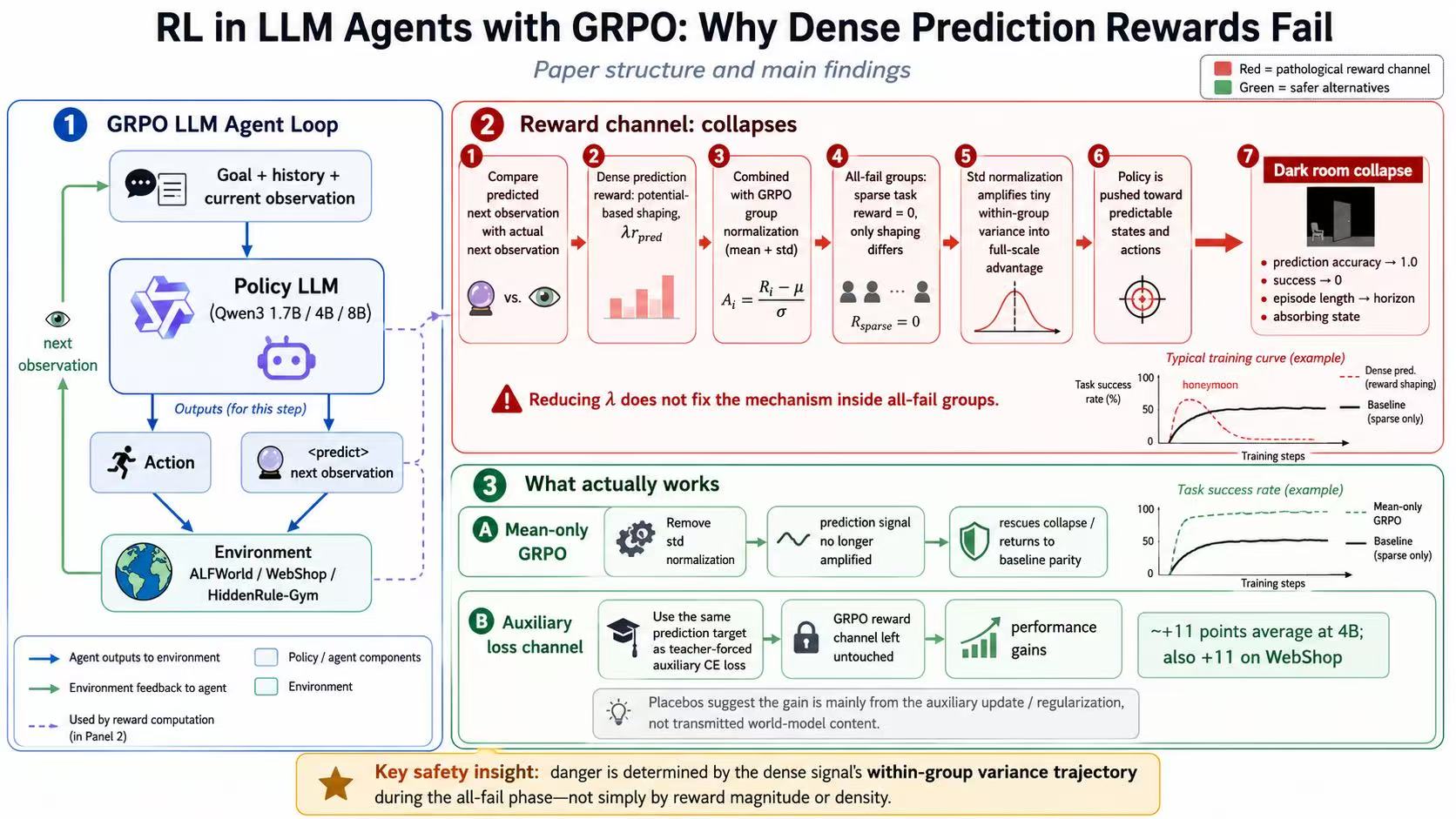}
  \caption{Overview of the study. \emph{Panel 1}: the step-independent GRPO
  agent loop with a verifiable \texttt{<predict>} block scored against the next
  observation (\S\ref{sec:setting}). \emph{Panel 2}: the failure cascade in the
  reward channel: in all-fail groups the $\varepsilon$-scale shaping term is
  the only within-group variation, std normalization amplifies it to
  full-scale advantages, and the policy converges to a dark-room absorbing
  state; reducing $\lambda$ does not change the update
  (\S\ref{sec:props}, \S\ref{sec:failure}). \emph{Panel 3}: the two
  deliveries that survive our controls: mean-only normalization (rescue to
  baseline parity) and the auxiliary-CE loss channel, whose gain the placebo
  ladder attributes to the update itself rather than the world-model content
  (\S\ref{sec:channel}). Annotated gains are rounded summaries: the 4B
  auxiliary gain measures $+14.2$ training-caliber / $+15.4$ formal on
  three-seed means (\S\ref{sec:e01}); the WebShop $+11.0$ is training-caliber
  and seed-split (formal $+3.0$, statistically level;
  \S\ref{sec:generalization}); and the loss channel's safety is
  scale-qualified by the 8B bistability (\S\ref{sec:generalization}). The
  boxed insight's ``all-fail phase'' is the sparse-success route into the
  variance-dominated window; a continuous-score second route exists
  (\S\ref{sec:variance-profile}).}
  \label{fig:overview}
\end{figure}

% P1 (hook): task + the surprising failure, one message.
Long-horizon LLM agents trained with sparse success rewards learn slowly. A natural
remedy is dense per-step supervision: ask the policy to predict its next observation
and reward it for being right. Some form of this recipe is on the menu wherever
agent RL meets sparse rewards, which today means most GRPO-family agent
training. We show the remedy backfires catastrophically under
group-normalized RL \citep{shao2024grpo}; Figure~\ref{fig:overview} maps the
setup, the failure cascade, and the two deliveries that survive our controls. The \emph{dark room problem}
\citep{friston2012darkroom} asks why a surprise-minimizing agent does not retreat to
a fully predictable corner; long a philosophical puzzle for predictive-processing
theories \citep{baltieri2019darkroom,seth2020curious}, it is a concrete engineering
outcome in our experiments: at every scale, group size, and fixed shaping
coefficient we tested, the optimizer builds the dark room automatically (timed
withdrawal inside the early gain window is the one schedule that escapes,
\S\ref{sec:satrace}). Curiosity-driven agents chase prediction \emph{error} and get
trapped by noisy TVs \citep{burda2019curiosity}; ours chase prediction
\emph{accuracy} and get trapped by their own predictability: the same failure
reached from opposite directions.
% P2 (challenge): why the failure is puzzling given the safety design.
The failure is puzzling because the signal meets every safety standard the shaping
literature asks for. In our step-independent multi-turn setting
\citep{feng2025gigpo}, the policy emits a verifiable \texttt{<predict>} block each
step, rule-scored against the actual next observation with no judge model; the
score enters as a difference-form per-step reward, superficially shaped like
potential-based shaping (PBRS), whose episode sum telescopes so return-level
distortion is provably at most $\lambda$. Yet every run that sustains the
difference-form reward under untouched std normalization in our sweep collapses.
On ALFWorld the natural history is stereotyped:
at 4B and 8B an early gain phase (a ``honeymoon,'' absent at 1.7B) whose peaks rise
with scale though never above the arm's formal baseline, then a turn, then an
absorbing state from which no intervention we tested recovers; on WebShop the
honeymoon is replaced by instant suppression or baseline parity
(\S\ref{sec:generalization}).
% P3 (root cause): the mechanism.
The root cause lives one level below returns, in the advantage estimator. Sparse
success makes all-fail groups common, and inside such a group the only within-group
variation is the $\varepsilon$-scale shaping term. GRPO's z-scoring divides by the
group's own standard deviation, so the normalized advantage is \emph{invariant to
the shaping coefficient} (Proposition~\ref{prop:lambda}): bounded rewards become
full-scale pressure, and, while the shaping term dominates the pool's residual spread
($\lambda\sigma_s \gg \max(\sigma_p, \epsilon)$, \S\ref{sec:props}), no reduction of the
coefficient short of zero changes the update. A three-point $\lambda$-sweep
confirms this (amplified advantage scale $\pm 7.3/\pm 9.1/\pm 7.2$ across a
tenfold range; every arm collapses), and a single-factor ablation completes the
causal case: removing only the std normalizer turns 0\% into baseline parity, and
a knock-out matrix (remove the signal, the amplifier, or the all-fail groups)
shows each ingredient is necessary.
% P4 (generalizing insight): the criterion.
The mechanism generalizes into the paper's organizing principle, a
\emph{variance-trajectory criterion} (endpoint case formalized in
Proposition~\ref{prop:variance}): what
z-scoring amplifies is a dense signal's within-group variance, so a signal's
danger is set by whether that variance survives the phase when it dominates the
group advantage (all-fail domination is the sparse-success route into that
phase; a continuous-score easy regime is a second route), not by magnitude or
density. The criterion retrodicts every
reward-channel collapse in this paper and every null but one (anchor-QA, where
it over-warns: variance persisted yet the formal endpoint is level with
baseline, \S\ref{sec:variance-profile}), makes preregistered
prospective predictions (full SHA256 digests committed to a public
version-controlled file before the runs completed, shipped with the code
release), and is compatible with published reward-channel successes. An
early-warning rule built on it detects all three held-out reward-channel
collapses 13--73 steps in advance, with at most one false alarm among the eleven
training-healthy arms (a set distinct from the abstract's eleven collapsing
runs; the flagged arm was itself later revised below baseline by the formal
evaluation, blurring even that alarm's status).

% P5 (the positive result and its attribution).
The criterion says when the reward channel is dangerous; a controlled
signal-delivery matrix says what to do instead, and why the popular explanation of
the fix is wrong. Holding signal, model, environment, and budget fixed and varying
only the consumption mechanism, every reward-channel variant on ALFWorld and
WebShop is at best neutral. The same signal as an auxiliary teacher-forced loss
does no harm, and what gain appears is not the signal's: under the identical
recipe a placebo ladder (shuffled gold, random-vocabulary,
random-token) matches or beats gold at both matched seeds (the placebo mean 78.8
vs.\ gold 68.6 at seed 0; the single s42 placebo endpoint 67.9 vs.\ 57.9), a
compute-matched no-target control shows no gain (48.6\%), and the
random-vocabulary placebo tops the entire table; a cross-normalization
convergence reinforces the reading descriptively (fixing the normalizer alone,
with no signal at all, reaches the same level, 57.9\% formal). Each placebo
family carries its own three-seed replication (at training caliber five of six
seeds above the
baseline band, the sixth inside it; none below, none collapsing), so the
class-level ordering does not rest on a single run. What works is the auxiliary
CE pass itself, acting as a regularizer; the world-model content plays no
measurable role, which bears awkwardly on the auxiliary-world-model literature,
whose gains may partly rest on the same effect.

% P5b (boundary): generalization and scale.
Both findings generalize within boundaries we map in \S\ref{sec:generalization}.
On WebShop the reward channel fails again, instantly in the easy regime and
through the familiar collapse-after-parity shape in the hard regime, while the
auxiliary channel is seed-split in the easy regime (single-seed-pair
$+11.0$/$-9.9$; formal $+3.0$ statistically level and $-14.0$; the hard-regime
cell is untested). At 8B the auxiliary recipe itself becomes \emph{bistable}:
gold full-weight training locks two of three seeds into an absorbing state and
sends the third to the study's best gold-signal result, while every content-free
or reduced-weight arm stays healthy. The randomness has moved from collapse
timing to collapse occurrence, and content-inertness, intact on the gain side,
reverses on the risk side.

% P6 (positioning + contributions).
We are not proposing that world-model objectives be added to agents; we report
what happens when they are, so the recipe under study is not itself a
contribution. Our contributions are:
\begin{enumerate}
  \item \textbf{Mechanism and causal localization of a catastrophic failure.}
  Three-scale collapse phenomenology; a one-line invariance result
  (Proposition~\ref{prop:lambda}) with three-point confirmation; a single-factor
  rescue and knock-out matrix; a behavioral characterization of the absorbing
  state; irreversibility across scheduling, coefficient scale, and post-hoc
  structural interventions (\S\ref{sec:analysis}, \S\ref{sec:failure}).
  \item \textbf{A signal-side safety criterion with prospective validation.} The
  variance-trajectory criterion (endpoint case: Proposition~\ref{prop:variance}), its
  retrodictions (complete on collapses; one over-warning on nulls), preregistered prospective tests, a validated early-warning
  signature, and a compatibility check against published successes
  (\S\ref{sec:variance-profile}, \S\ref{sec:failure}).
  \item \textbf{The channel effect, fully attributed.} A fourteen-arm delivery
  matrix (the same signal across every signal-bearing arm, plus matched
  controls) and placebo ladder attributing the auxiliary-loss gain to the
  update itself, with seed replication and a controlled challenge to
  world-modeling attributions in prior work (\S\ref{sec:channel}).
  \item \textbf{Scale and environment boundaries of both results.}
  Cross-environment replication of the failure and the content-free attribution
  on WebShop (where the auxiliary gain is seed-split), a regime-dependent
  failure-form taxonomy (instant suppression vs.\ collapse-after-parity), and
  the 8B bistability: same recipe, seed-determined locked collapse or the
  study's best gold-signal gain (\S\ref{sec:generalization}).
\end{enumerate}

\section{Related Work}
\label{sec:related}

\textbf{Group-normalized RL and its pathologies.} GRPO
\citep{shao2024grpo,guo2025deepseekr1} normalizes group returns by mean and
std; Dr.GRPO \citep{liu2025drgrpo} identified the std term's biases, and we
adopt its mean-only variant as an \emph{instrument}, contributing the causal
demonstration that in long-horizon sparse-success agents the std term alone
separates catastrophic from benign. Our results delimit where the
std-as-adaptive-gradient defense \citep{ge2026curvature} breaks: phases where
an $\varepsilon$-scale shaping term supplies the dominant within-group
variance, with sparse-success all-fail groups as the extreme case. A distinct
line treats degenerate groups as zero-signal dead zones to prune or reweight
\citep{xu2025spo,zhang2026aero}; our setting inverts that premise (all-fail
groups become \emph{full-signal} zones, Proposition~\ref{prop:lambda}), and
remedies that merely skip them remove the amplifier's fuel without the
amplifier, rescue only partially ($-9.9$pt, \S\ref{sec:rescue}), and can mask
the mechanism. Proposition~\ref{prop:lambda} is, to our knowledge, the first
formalization of this amplification for the sparse-success all-fail structure
of multi-turn agents.

\textbf{Dense signals and shaping.} Prediction-\emph{error} seeking falls to
the noisy-TV trap \citep{pathak2017curiosity,burda2019curiosity}; our failure
is the mirror image, prediction-\emph{accuracy} seeking parking in predictable
corners \citep{friston2012darkroom}. Ng-style invariance \citep{ng1999policy}
assumes a policy-independent potential; the known ways it breaks
\citep{devlin2012dynamic,behboudian2021pies} never consider group
normalization, where our damage concentrates. The standard hacking mitigation
is the progress principle \citep{setlur2024pav,hou2025lpm}; our $\Delta$acc
arm is its preregistered test, and its failure sharpens the boundedness
principle of \citet{fu2025par}: return-level boundedness does not survive
group z-scoring at the step level.

\textbf{Auxiliary objectives and world modeling.} Loss-channel world-model
signals have a consistent track record \citep{jaderberg2017unreal,
shrivastava2026echo,lu2026paw}; the reward channel is not uniformly hostile
(RWML \citep{yu2026rwml}, VAGEN \citep{wang2025vagen}), and our criterion
must be, and is, consistent with those successes (\S\ref{sec:variance-profile}).
What this literature lacks is a controlled comparison holding the signal fixed
while varying only the consumption mechanism; our matrix supplies it
(\S\ref{sec:channel}), and relative to ECHO \citep{shrivastava2026echo} we
contribute the attribution. Our 8B bistability is not generic seed variance
\citep{dodge2020finetuning} rebranded (\S\ref{sec:generalization}). Extended
related work, including grouping-scheme successors, qualitative precedents of
the amplification, and scale-instability context, is in
Appendix~\ref{app:related-ext}.

\section{Method: Signal, Amplification, and the Variance Criterion}
\label{sec:analysis}

Three modules organize the paper's apparatus: the verifiable prediction signal
and how it enters the reward (\S\ref{sec:setting}); the amplification algebra
that turns a bounded signal into full-scale pressure
(\S\ref{sec:props}); and the variance-trajectory criterion that generalizes it
(\S\ref{sec:variance-profile}).

\subsection{Setting: verifiable prediction as a shaping reward}
\label{sec:setting}

\paragraph{Agent framework.}
Step-independent multi-turn rollout \citep{feng2025gigpo}: each step's LLM input is built
from the current observation plus a bounded history summary, keeping context length
near-constant over 15--50 step horizons (per-environment, Appendix~\ref{app:config}). Training uses GRPO \citep{shao2024grpo} with
grouped environments (group size 4) and terminal task reward (ALFWorld:
sparse success 10 / failure 0; WebShop replaces the binary outcome with a
continuous terminal matching score in $[0,1]$, the ``continuous scores'' of
\S\ref{sec:variance-profile}), plus an invalid-action penalty.

\paragraph{Prediction-sufficiency (PS) signal.}
Prompts request a \texttt{<predict>} block over a task-agnostic feature schema $\Phi$
(location, objects-visible boolean, receptacle state; an open-set visible-objects F1 is
logged but never rewarded). The block is rule-parsed and verified against the \emph{next}
environment observation; no judge model is involved. The per-step reward is the difference form
$\rpred(t) = \Phi_t - \Phi_{t-1}$ scaled by $\lambda{=}0.1$, injected on the last response
token of each step sample.

\subsection{Bounded returns, unbounded advantages}
\label{sec:props}

Let a group contain $G$ trajectories with returns $R_i$; GRPO's normalized advantage and
our per-step shaping reward are
\begin{equation}
\hat{A}_i \;=\; \frac{R_i - \bar{R}}{\sigma_R + \epsilon},
\qquad
\rpred(t) \;=\; \lambda\,\bigl(\Phi_t - \Phi_{t-1}\bigr),
\label{eq:defs}
\end{equation}
where $\Phi_t \in [0,1]$ scores the policy's own prediction against the realized next
observation. Because the shaping telescopes,
\begin{equation}
\sum_t \rpred(t) \;=\; \lambda\,\bigl(\Phi_T - \Phi_0\bigr), \qquad
\Bigl|\sum_t \rpred(t)\Bigr| \;\le\; \lambda,
\label{eq:telescope}
\end{equation}
episode returns are distorted by at most $\lambda$: at the \emph{return} level the
arithmetic bound delivers the safety that PBRS rhetoric suggests. The formal PBRS
invariance theorem \citep{ng1999policy} does not apply, however: $\Phi$ scores the
policy's \emph{own} predictions, a policy-dependent potential, exactly the class
for which invariance is known to break (\S\ref{sec:related}). The compliance is
superficial, and the failure enters one level down.

\paragraph{The normalization unit.}
One implementation fact is load-bearing for everything that follows. Each of a
trajectory's $T$ steps is a separate sample in the update batch, and the GRPO
estimator computes the group mean and std over \emph{all step-samples sharing a
group ID} (its cross-step default), not over the $G$ trajectory returns: the pool
has $n = G \cdot T \approx 4 \times 50 \le 200$ members, each scored by its own
per-step reward, and each sample's z-score is broadcast to its response tokens.
The propositions below use trajectory notation for readability, but the algebra
is unit-agnostic and the unit matters quantitatively: a z-score over $n$ pooled
values is bounded by $(n-1)/\sqrt{n}$, i.e.\ $1.5$ at $n{=}4$ but ${\approx}14$
at $n{=}200$. The advantage magnitudes we measure ($\pm 5$--$9$,
\S\ref{sec:failure}) are unreachable under trajectory pooling and sit inside the
step-sample bound; the fingerprint itself identifies the unit. The unit also
dissolves an apparent tension with Eq.~\ref{eq:telescope}: telescoping bounds the
return-level distortion, but the normalizer never sees that sum; it divides by
the pooled spread of the \emph{per-step} values $\lambda(\Phi_t - \Phi_{t-1})$,
which keeps fluctuating as long as the policy visits varied states. A shaping
term can therefore be return-bounded and advantage-dominant at once. One more $\varepsilon$-scale term shares the pool: the invalid-action
penalty ($0.1$, the same order as $\lambda$). Proposition~\ref{prop:lambda}
states both limits; the $\lambda$-sweep alone cannot apportion the pooled
spread, and the causal apportioning (every baseline arm carries the same
penalty under the same normalizer and the same early all-fail exposure and
never collapses; the terminal behavior is prediction-optimal, the imprint of
shaping pressure during the approach) assigns the operative pressure to the
shaping component. The full identifiability discussion, including the
implementable-but-unlogged spread monitor, is in
Appendix~\ref{app:criterion-detail}.

\begin{proposition}[$\lambda$-invariance in all-fail groups]
\label{prop:lambda}
In a group where every trajectory fails the task, write
$s_i = \sum_t (\Phi_t - \Phi_{t-1})_i$ for the accumulated shaping potential and
$p_i$ for the accumulated invalid-action penalty, so $R_i = C + \lambda s_i - p_i$
with $C$ the shared task return. If the shaping term dominates the within-group
spread ($\lambda\sigma_s \gg \max(\sigma_p, \epsilon)$),
\begin{equation}
\hat{A}_i \;=\; \frac{\lambda\,(s_i - \bar{s}) - (p_i - \bar{p})}{\sigma_{\lambda s - p} + \epsilon}
\;\longrightarrow\;
\frac{s_i - \bar{s}}{\sigma_s},
\label{eq:lambdainv}
\end{equation}
which is independent of $\lambda$. In the opposite limit
($\sigma_p \gg \max(\lambda\sigma_s, \epsilon)$) the advantage tends to
$-(p_i - \bar{p})/\sigma_p$, which contains no $\lambda$ either; only a balanced
mixture ($\lambda\sigma_s \sim \sigma_p$) admits a $\lambda$ trend.
\end{proposition}

\begin{proof}
Substitute $R_i = C + \lambda s_i - p_i$ into Eq.~\ref{eq:defs} and take the
respective limits.
\end{proof}

Proposition~\ref{prop:lambda} says the z-score \emph{cancels the shaping
coefficient}: an $\varepsilon$-scale signal is stretched to full z-scale
advantages per pooled step-sample (the measured $\pm 5$--$9$ fingerprint, within
the ${\approx}14$ bound above). \emph{Within the
$\lambda\sigma_s \gg \max(\sigma_p, \epsilon)$
regime}, no reduction of $\lambda$ changes the update; once
$\lambda\sigma_s \sim \max(\sigma_p, \epsilon)$ the shaping term no longer
controls the pool and the first limit gives way,
so the statement is conditional, not absolute. With $\epsilon = 10^{-6}$ (the
framework default), the regime holds while the signal retains within-group
variance; at saturation $\sigma_s \to 0$ and the shaping contribution vanishes
(Proposition~\ref{prop:variance}). This one line predicts two otherwise separate
facts: the anneal arm collapsing unchanged despite $\lambda$ having already
fallen to ${\approx}0.068$ by the turn
(\S\ref{sec:satrace}), and the registered prediction that std runs behave
identically across $\lambda \in \{0.01, 0.03, 0.1\}$. The sweep confirms all
three points: amplified scales $\pm 7.3 / \pm 9.1 / \pm 7.2$ (no trend over a
tenfold range), collapse in every arm (endpoints 1.0/0.0/1.0\%), timing within
the variance band. The coefficient is cancelled; what remains is the variance
structure.

\begin{proposition}[Saturation-endpoint case of the variance criterion]
\label{prop:variance}
In the setting of Proposition~\ref{prop:lambda} (an all-fail group with
$R_i = C + \lambda s_i - p_i$; the dominance premise is not inherited), if the
signal saturates
so that its within-group deviations vanish on the scale of the remaining spread
($\lambda\sigma_s \ll \max(\sigma_p, \epsilon)$), the \emph{shaping
component's} contribution to the advantage vanishes,
\begin{equation}
\frac{\lambda\,(s_i - \bar{s})}{\sigma_{\lambda s - p} + \epsilon}
\;\longrightarrow\; 0 ,
\label{eq:starve}
\end{equation}
and the amplifier can no longer feed on the signal. If the penalty spread also
vanishes ($\sigma_p \to 0$), the $\epsilon$-floor reactivates and the amplifier
starves outright; otherwise the z-score is taken over by the penalty term
(the second limit of Proposition~\ref{prop:lambda}), amplified pressure that
rewards surface compliance rather than the signal, the same pressure every
healthy baseline carries. In either case the danger a dense signal itself
contributes under std normalization is governed not by its magnitude or density
but by its \emph{within-group variance at saturation}.
\end{proposition}

\subsection{The variance-trajectory criterion}
\label{sec:variance-profile}
What z-scoring amplifies is the signal's within-group \emph{variance}, so we state
the danger condition over the whole variance trajectory, not only its endpoint:
\textbf{a dense signal is dangerous under std normalization to the extent that it
sustains within-group variance during the phase when that variance dominates
the group advantage}. All-fail domination is the sparse-success route into that
phase (our collapse configuration); the WebShop easy regime shows a second
route, where continuous task scores over short episodes let prediction variance
dominate from step~0 (\S\ref{sec:generalization}). The window is not defined by
the harm it produces, but our ex-ante observability is uneven and we state it
exactly. The all-fail group fraction was measured live in the two group-size
ablations (its climb precedes their literal-zero windows) and retrodicted for
the original run (\S\ref{sec:failure}); the shaping term's share of the pooled
step-sample spread is implementable as a monitor (the terms enter the reward
separately) but our runs logged only batch means (\S\ref{sec:props}); and for
the continuous-score route neither proxy applies, so the warning for that
regime (\S\ref{sec:guidance}) rests on the measured failure itself, not on a
monitored precursor.
For saturating signals the endpoint ($\sigma_{\text{within}}$ at mastery) is the
limiting summary, but the operative quantity is the variance carried through the
variance-dominated window; a progress-style signal has its \emph{largest}
variance while the agent is still learning, so the criterion does not declare it
safe during training, only that its pressure decays as mastery approaches.
Empirically the signal forms separate as the criterion predicts:
difference-form variance persists to the absorbing state (collapse); saturated
always-positive confidence starves the signal-fed amplification (no collapse);
progress-style $\Delta$acc plateaus below saturation and keeps paying (chronic
drag: the endpoint form of the criterion would have called it safe, the
trajectory form does not); anchor-QA over-warns (drag at training caliber
only). A matched-variance noise control adds the second axis: sustained
variance plus \emph{hackability} is necessary in every reward-channel lock we
observe but not jointly sufficient. The full walkthrough is in
Appendix~\ref{app:criterion-detail}.
The criterion is consistent with published reward-channel successes in the
weak sense that neither contradicts it (RWML is level-scored, outside the
difference-form family; VAGEN's estimator coordinates are unreported, and our
GiGPO arm shows redesign alone does not decide); it issues no safety
certificates, and the full compatibility discussion is in
Appendix~\ref{app:criterion-detail}.
Proposition~\ref{prop:variance} formalizes the criterion's \emph{endpoint}
case; the trajectory form, which tracks $\sigma_s$ through the
variance-dominated window rather than only at saturation, is its empirical
strengthening, the one the $\Delta$acc and anchor-QA arms forced. The
saturation quantity is the same one that lower-bounds the policy-gradient
magnitude in \citet{leng2025mmr1}. This is a
signal-side taxonomy, orthogonal to algorithm-side variance work that redesigns
the normalizer or its weights \citep{xiao2025bnpo,tao2025hero}; its tests were
preregistered before the runs completed.

\section{Experimental Setting}
\label{sec:exp-setup}

\paragraph{Environments and framework.}
Three environments carry results, all driven by the step-independent multi-turn
agent stack of \S\ref{sec:setting} on the open-source verl-agent framework
(hardware and per-environment configuration in Appendix~\ref{app:config}).
ALFWorld \citep{shridhar2021alfworld} is the primary testbed (sparse terminal
success 10 / failure 0; horizon 50); WebShop \citep{yao2022webshop} replicates
the matrix core under a \emph{continuous} terminal matching score in two
regimes (small: 1k products, baseline ${\sim}60\%$; full: 1.18M products with
human goals, baseline ${\sim}29\%$; horizon 15); HiddenRule-Gym is our
synthetic rooms-and-devices POMDP, described below.

\paragraph{Models and configuration.}
Qwen3-1.7B/4B/8B, identical configuration across scales (train batch 8 $\times$ group 4
$=$ 32 trajectories/step; validation 32 episodes, 16 for HiddenRule-Gym at 1.7B
with the 4B HRG capacity arm at 32, sampled decoding at temperature 0.4).
Endpoints are last-6 validation means (single-point binomial noise up to
$\pm 8.8$pt; last-6 SE $\pm 3.4$pt from the measured single-point $\sigma$ of
8.4), seed 0 unless stated (the headline auxiliary-loss arm, both leading
placebos, and the ALFWorld baseline carry seeds 0/42/96); formal endpoint numbers from a unified 140-game
evaluation of every arm entering a cross-arm ordering claim are reported in \S\ref{sec:e01}
(Table~\ref{tab:e01}), and prose keeps last-6 numbers where it describes training
dynamics.

\paragraph{Budget.}
All arms run at one quarter of published full-budget ALFWorld training (32
trajectories/step, 150 steps) so fourteen mutually comparable arms fit one
compute envelope; every claim is a same-budget arm-to-arm comparison and none
depends on absolute success rates (published baselines, a full-budget sanity
anchor at 67.2\%, and state-of-the-art reference points are in
Appendices~\ref{app:mech-detail} and~\ref{app:sota}).

\paragraph{HiddenRule-Gym (HRG).}
A synthetic rooms-and-devices POMDP with four hidden-rule families (conjunction, sequence,
XOR, count; train/probe family split), a BFS oracle sharing the environment's pure
transition function, and (uniquely) \emph{exactly computable} feature coverage
$C = I(\Phi;s)/H(s)$ over non-terminal reachable states, with a greedy mask ladder giving
measured coverage levels for the sweep in \S\ref{sec:hrg}. Privileged-latent belief probes
carry a 5-gram leakage audit.

\paragraph{Arms, baselines, and evaluation calibers.}
Every claim is a same-budget arm-to-arm comparison against a matched GRPO
baseline trained under the identical configuration (three seeds at 4B on
ALFWorld; per-scale and per-environment baselines elsewhere); the study
comprises 74 training arms, inventoried with endpoints in
Appendix~\ref{app:inventory}, and its predictions were preregistered by SHA256
digest in a public version-controlled file before the corresponding runs
completed (ledger in Appendix~\ref{app:prereg}). Two evaluation calibers
recur throughout: \emph{training caliber} (the last-6 validation means above,
used for training dynamics) and \emph{formal endpoints} (the unified 140-game
re-evaluation of \S\ref{sec:e01}, used for every cross-arm ordering claim;
WebShop uses the analogous 200-goal protocol).

\section{Experimental Results}
\label{sec:failure}

\subsection{Three-scale collapse and the mechanism, measured}

\begin{wrapfigure}{r}{0.5\linewidth}
  \vspace{-1.1\baselineskip}
  \centering
  \includegraphics[width=\linewidth]{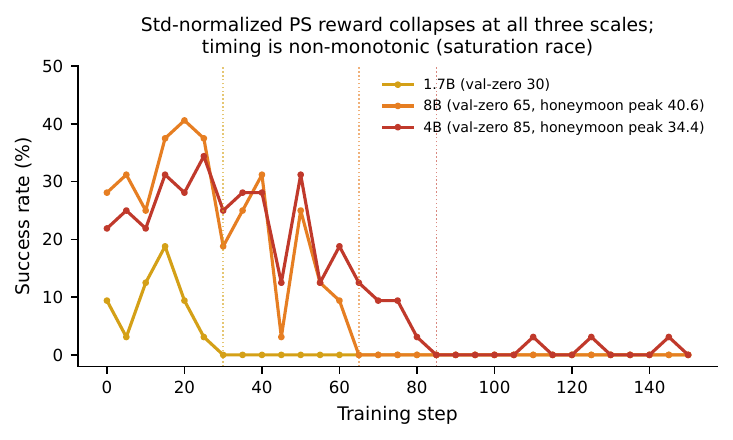}
  \caption{The std-normalized prediction reward collapses every run at all three scales;
  timing is non-monotonic (saturation race, \S\ref{sec:satrace}). Honeymoon peaks
  \emph{rise} with scale: the signal lifts each arm's own early trajectory until
  hacking pressure arrives. Neither honeymoon peak exceeds its arm's formal
  baseline (8B $40.6 < 45.0$; 4B $34.4 < 42.9$); the 1.7B curve's single
  transient point ($18.8$, not a honeymoon by the Table~\ref{tab:collapse}
  criterion) sits at its formal baseline's level ($17.9$).}
  \label{fig:three-scale-curves}
  \vspace{-0.9\baselineskip}
\end{wrapfigure}

Sparse reward $\to$ all-fail groups $\to$ the only within-group differences are
$\varepsilon$-scale $\lambda\,\rpred$ $\to$ std normalization amplifies them to full scale
$\to$ the gradient rewards \emph{predictability} $\to$ the policy drifts toward familiar
states $\to$ fewer successes $\to$ stronger domination: a self-reinforcing loop, and at
saturation $\rpred\to 0$ strands the policy in the degenerate basin.
The advantage fingerprint makes the amplification visible: all-fail-group advantages reach $\boldsymbol{\pm 5}$--$\boldsymbol{9}$ under std normalization (inside the pooled step-sample bound of ${\approx}14$, \S\ref{sec:props}) versus the unamplified $\pm 0.1$ scale ($=\lambda$) without it.
Replay forensics on a collapsed checkpoint, constant-budget group-size
ablations ($n{\in}\{4,8,16\}$: literal zero at steps 85/135/80, non-monotone),
and the all-fail-fraction telemetry (a climb to a terminal lock at $1.00$
coinciding with the literal-zero window) confirm each link of the loop, and a
tenfold $\lambda$ reduction leaves the amplified scale unchanged ($\pm 7.2$
vs.\ $\pm 7.3$), as Proposition~\ref{prop:lambda} predicts;
Figure~\ref{fig:groupsize-runaway} shows the group-size ablations and the
all-fail lock, and full measurement detail is in
Appendix~\ref{app:mech-detail}.

\begin{table}[H]
\centering
\begin{tabular}{lcccc}
\toprule
scale & honeymoon peak & turn & val-zero from & terminal triple \\
\midrule
1.7B & --- & ${\sim}15$--$20$ & 30 & pred $1.000$ / len 50 / succ 0 \\
4B   & 34.4 @25 & ${\sim}45$ & 85 & same \\
8B   & \textbf{40.6 @20} & ${\sim}48$ & 65 & same \\
\bottomrule
\end{tabular}
\caption{Collapse phenomenology. \emph{Turn} $=$ behavioral inflection (validation
begins its monotone decline); \emph{val-zero} $=$ first checkpoint at literal 0\%
success. These are distinct quantities, and the saturation-race ordering in
\S\ref{sec:satrace} refers to val-zero. Baselines for reference (last-6, log-exact):
1.7B 20.9\%, 4B 49.5\%, 8B 32.8\%. \emph{Honeymoon} is operationalized as at least
two consecutive validation points above the matched baseline's contemporaneous
curve; the 1.7B arm's single transient point (18.8\% at step 15, inside the
${\pm}8.8$pt single-point noise) does not qualify, hence ``absent'' there.}
\label{tab:collapse}
\end{table}

\begin{figure}[t]
  \centering
  \includegraphics[width=\linewidth]{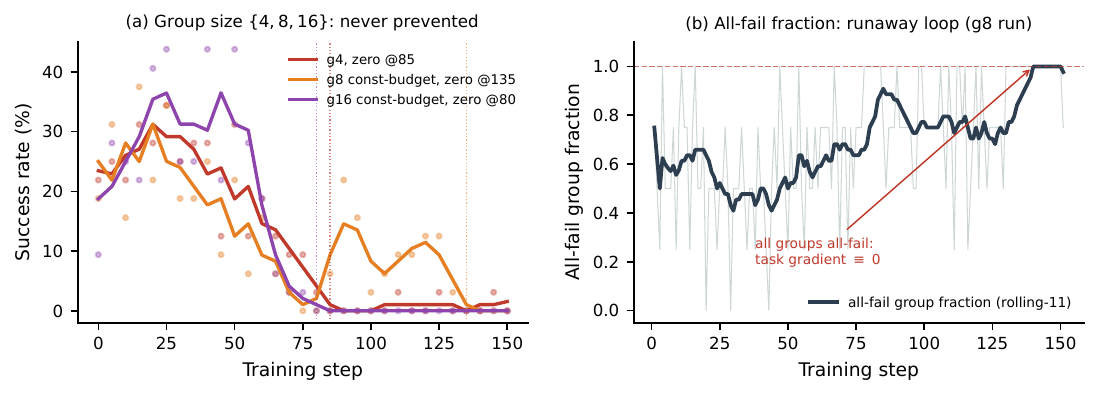}
  \caption{Left: constant-budget group-size ablations. Collapse occurs at every group size in
  $\{4,8,16\}$ with non-monotone timing (literal zero at steps 85/135/80): group size
  neither prevents nor reliably delays. Right: the all-fail group fraction along the
  $n{=}8$ run (rolling mean over 11 steps); its terminal lock at $1.00$ coincides with
  the literal-zero validation window: the absorbing state of the self-reinforcing loop measured
  directly (the $n{=}16$ run shows the same lock from step ${\sim}75$).}
  \label{fig:groupsize-runaway}
\end{figure}

\subsection{Saturation race, irreversibility, and early warning}
\paragraph{A saturation race, not a coefficient response.}
\label{sec:satrace}
We registered ``weaker model collapses earlier, 8B latest'' and the data falsified it:
1.7B@30 $<$ \textbf{8B@65} $<$ 4B@85 (val-zero caliber throughout; Table~\ref{tab:collapse}). Revision: collapse timing is a \emph{race} between
prediction-saturation speed and task-learning speed; the 8B model learns everything faster,
including how to be predictable. The honeymoon peak rising with scale (34.4 $\to$ 40.6)
shows the early transient scales with capacity; note it lifts the arm's own
trajectory only, and neither honeymoon peak exceeds its arm's formal baseline
(the 8B peak tops the artifact-depressed training-time 32.8, not the formal
45.0; the 1.7B single transient, no honeymoon by the Table~\ref{tab:collapse}
criterion, sits at its formal baseline's level, 18.8 vs.\ 17.9).

\paragraph{Irreversibility.}
Nothing recovers a collapsed policy: 120 post-collapse steps at fixed
$\lambda$ (1.7B), cosine annealing to zero (the turn arrives at
$\lambda{\approx}0.068$, step ${\sim}57$; the final ${\sim}30$
effectively-pure-GRPO steps show zero recovery), and a preregistered post-hoc
normalizer switch (resuming the seed-42 arm 70 steps inside its absorbing
state: 11 checkpoints at literal 0\%, the amplifier verifiably gone,
magnitudes $\pm 5$--$9 \to {\approx}0.1$) all fail; post-collapse entropy
\emph{rebounds} at all scales (0.39/0.44/0.52 ordered by reference-model
entropy), so the absorbing state is behavioral, not temperature-level. Three
preventive routes work: mean-only from the start, all-fail-group filtering
(partial, 39.6\%), and timed withdrawal (a fast-anneal arm completing inside
the honeymoon window survives at 40.1\%, Fig.~\ref{fig:collapse-fixes}).
Irreversibility binds after the turn, not before; full detail in
Appendix~\ref{app:mech-detail}.

\paragraph{Early-warning signature.}
A frozen conjunction rule (entropy-has-declined $\wedge$ prediction
${\ge}0.95$ $\wedge$ length ${\ge}0.95\,H$; thresholds calibrated only on the
four definition-set collapse runs) detects all three held-out
\emph{reward-channel} collapses \textbf{73/53/13 steps ahead} of the collapse point,
with at most one false alarm among the eleven training-healthy arms; the full
rule, roster, three-caliber definitions, and independence caveats are in
Appendix~\ref{app:mech-detail}.

\subsection{Single-factor rescue and the fix hierarchy}
\label{sec:rescue}

\begin{figure}[t]
  \centering
  \includegraphics[width=\linewidth]{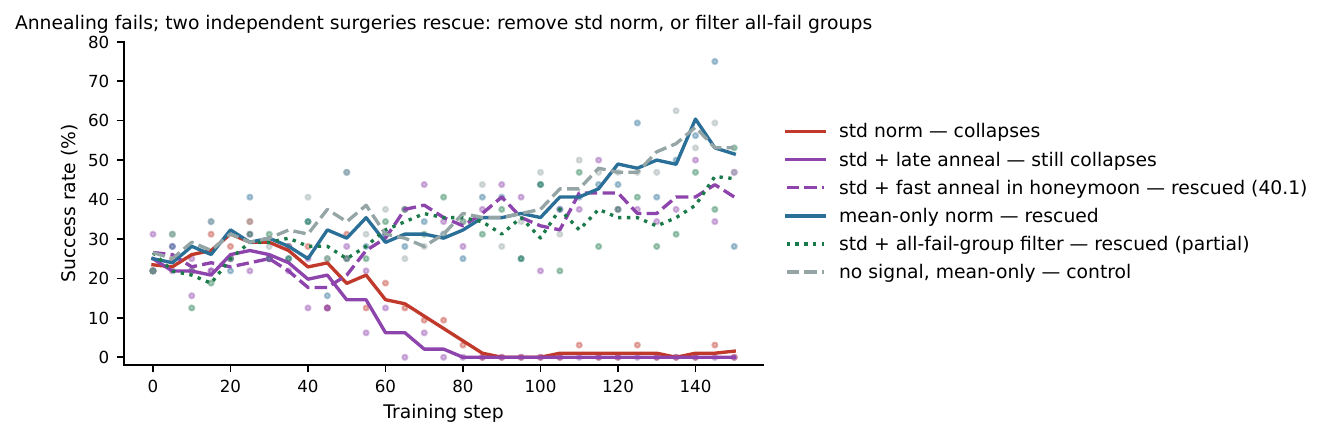}
  \caption{Two independent interventions rescue. The collapsed arm and the mean-only arm
  differ only in \texttt{norm\_adv\_by\_std\_in\_grpo}; the filter arm (dotted) keeps std
  normalization and instead drops all-fail groups from the gradient. At step 85, the
  unfiltered arm's literal-zero step, it sits at 31.2\%, and it finishes at 39.6\%
  (last-6). Late annealing does not rescue, but the same anneal applied \emph{inside}
  the honeymoon window does (dashed, 40.1\%): withdrawal timing relative to the
  honeymoon decides. The no-signal control matches the mean-only arm
  (net attribution $\approx 0$).}
  \label{fig:collapse-fixes}
\end{figure}

\paragraph{The rescue.}
Flipping one flag (advantages centered by group mean but not divided by group std) turns
0\% into \textbf{51.6\%} ($\geq$ baseline 49.5\%), peak 75.0, no terminal triple, prediction
accuracy does not saturate (0.49--0.64). This is the causal localization: the std term is
necessary for the catastrophe.

\paragraph{The second intervention.}
A preregistered companion arm keeps std normalization untouched and instead drops
all-fail groups from the gradient (RAFT-style filtering; 23--77\% of samples
zeroed per step). It also prevents collapse (no terminal triple; prediction
accuracy oscillates at 0.6--0.8) and finishes at \textbf{39.6\%}: rescued from
0\% but ${\sim}10$ points below baseline, consistent with zeroing a large,
fluctuating share of the samples (23--77\% per step) and with residual pressure
inside mixed groups. The two interventions
complete an ALFWorld knock-out matrix: remove the signal (baseline), the
amplifier (mean-only), or the all-fail groups (filter) and training is healthy;
only the full triple collapses. (On WebShop, where the terminal score is
continuous and strict all-fail groups do not form in either mode, the all-fail
leg does not apply; prediction variance dominates through low return spread
instead, \S\ref{sec:variance-profile}.) The asymmetry (mean-only restores parity, filtering
rescues partially) makes mean-only the preferred prescription, the filter arm
independently confirming that the pathology lives in the amplification of
all-fail groups, the same groups DAPO-style dynamic sampling discards for
variance reasons \citep{yu2025dapo}.

\paragraph{Net attribution.}
The control (mean-only normalization, \emph{no} prediction signal) lands at
\textbf{52.6\%}, dead even with 51.6\% ($-1.0$pt, far inside single-seed validation
noise; descriptive comparison): the prediction reward's net contribution in the
mean-normalized channel is $\approx 0$, and the rescue belongs to the
normalizer. Secondary finding: mean-only $\geq$ std baseline independently
replicates Dr.GRPO in long-horizon agents, with a caliber-dependent margin
($+3.1$ training on the matched-seed pair, 52.6 vs.\ 49.5; formally $+13.8$
against the three-seed mean and $+15.0$ on the matched seed-0 pairing, 57.9 vs.\
44.1/42.9). Residual: the PS arm sharpens
${\sim}20$ steps earlier at equal endpoint (``speed, not height''; single seed).

\paragraph{Two further arms sharpen the criterion and close the fix list.}
A preregistered matched-variance \emph{noise} control (random potentials, same
per-step variance, not policy-controllable) drags severely (23.4\%, $-26$ vs.\
baseline) yet never locks: hackability is the criterion's second axis, and the
two axes are necessary in every observed reward-channel lock but not jointly
sufficient (the $\Delta$acc and anchor-QA arms satisfy both and neither
locks). Per-channel \emph{decoupling} (GDPO-style, contribution capped at
$\lambda$) does not collapse but under-performs (31.2/42.2 last-6; formal
27.9/32.1, both below matched baselines): the per-channel z-score is a capped
copy of the same amplifier, and advantage \emph{magnitude} does not order the
outcomes, amplification structure does (the harmless mean-only arm's raw
$\pm 0.098$ exceeds the dragging capped arm's $0.066$--$0.077$).
Fig.~\ref{fig:dose} plots magnitude against outcome; full paragraphs are in
Appendix~\ref{app:fix-detail}.

\begin{figure}[t]
  \begin{minipage}[t]{0.42\linewidth}
    \centering
    \includegraphics[width=\linewidth]{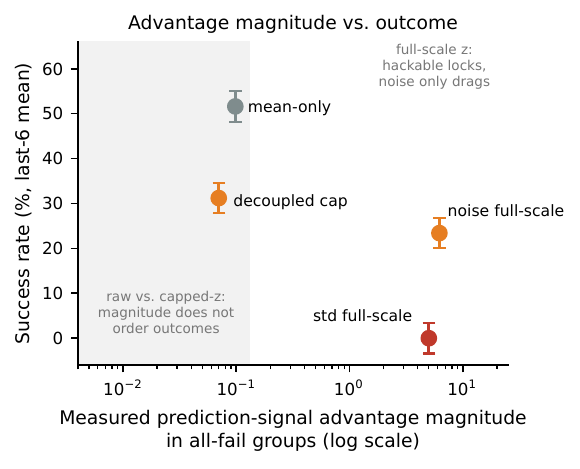}
    \caption{Advantage magnitude versus outcome in all-fail groups (log $x$).
    Magnitude does not order the outcomes: the harmless mean-only arm's raw
    contribution ($\pm 0.098$) exceeds the dragging decoupled arm's capped one
    (${\approx}0.07$). The split tracks amplification structure (raw / capped
    z-scored / full-scale z-scored) and, at full scale, hackability
    (matched-variance noise drags but does not lock).}
    \label{fig:dose}
  \end{minipage}\hfill
  \begin{minipage}[t]{0.55\linewidth}
    \centering
    \includegraphics[width=\linewidth]{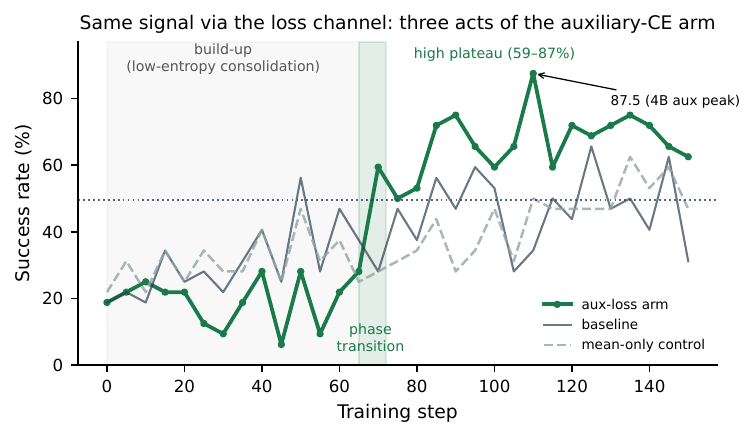}
    \caption{The auxiliary-loss arm's three phases: low-entropy build-up
    (entropy 0.04, a false ``erosion'' alarm), phase transition (steps 65--72,
    28${\to}$59\%), high plateau (59--87\%; 4B aux-arm peak 87.5). The
    interference probe (task-batch log-prob shift across the auxiliary update)
    stays at zero throughout: the auxiliary gradient coexists cleanly with the
    task gradient.}
    \label{fig:aux-headline}
  \end{minipage}
\end{figure}

\subsection{The channel effect}
\label{sec:channel}

\begin{figure}[t]
  \centering
  \includegraphics[width=0.8\linewidth]{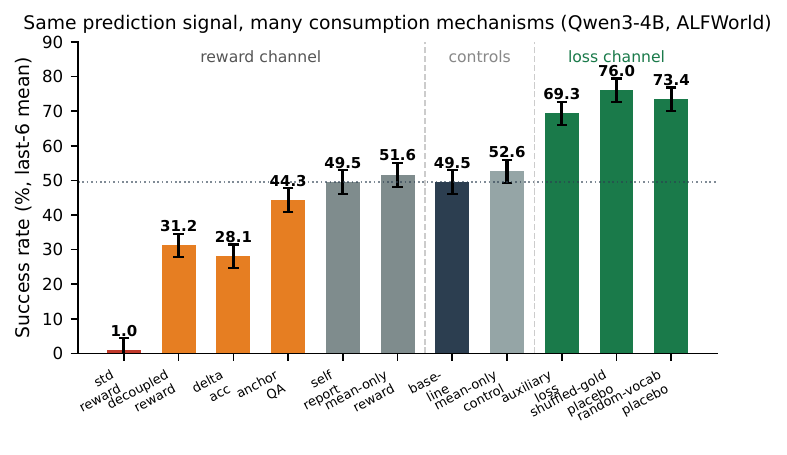}
  \caption{The signal-delivery matrix (11 of the 14 arms of
Table~\ref{tab:channel} shown; the all-fail filter, compute-matched replay, and
format-stripped placebo omitted for width): same
prediction signal, varying consumption mechanism and normalization, identical
environment/model/budget. Regime membership tracks the delivery channel rather than the signal form. Numbers here are training-time last-6 means (the quantity the
  dynamics narrative tracks); formal 140-game endpoints for every arm are in
  Table~\ref{tab:e01}.}
  \label{fig:channel-matrix}
\end{figure}

\begin{table}[H]
\centering
\scriptsize
\begin{tabular}{llllr}
\toprule
arm & channel & signal form & normalization & last-6 \\
\midrule
std reward & reward & diff.-form & std & 1.0\% \\
decoupled reward & reward & diff.-form, per-channel cap & std (decoupled) & 31.2\% \\
$\Delta$acc progress & reward & clipped accuracy progress & std & 28.1\% \\
mean-only reward & reward & diff.-form & mean-only & 51.6\% \\
all-fail filter & reward & diff.-form & std, all-fail groups dropped & 39.6\% \\
mean-only control & --- & \emph{no signal} & mean-only & 52.6\% \\
self-report & reward & always-positive confidence & std & 49.5\%\footnotemark \\
anchor-QA & reward & recall accuracy & std & 44.3\% \\
baseline & --- & --- & std & 49.5\% \\
compute-matched replay & --- & no aux target, $2\times$ epochs & std & 48.9\% \\
\textbf{auxiliary loss} & \textbf{loss} & teacher-forced CE on gold predict & std untouched & \textbf{69.3\%} \\
\textbf{shuffled-gold placebo} & \textbf{loss} & CE on permuted gold & std untouched & \textbf{76.0\%} \\
\textbf{random-vocabulary placebo} & \textbf{loss} & CE on random out-of-domain strings & std untouched & \textbf{73.4\%} \\
\textbf{random-token placebo} & \textbf{loss} & CE on bare random words (format-stripped) & std untouched & \textbf{61.5\%} \\
\bottomrule
\end{tabular}
\caption{The signal-delivery matrix. On class means the loss-channel arms beat every
reward-channel variant with the std reward channel left untouched; the shuffled-gold
placebo matching the true-gold arm shows exact content--context pairing is not the active
ingredient (see text). The auxiliary-loss arm carries three seeds (69.3/63.5/50.0, mean
60.9), and both leading placebos now carry three seeds of their own
(shuffled 76.0/60.4/48.9; random-vocabulary 73.4/77.6/60.9, training-caliber):
five of the six placebo seeds land above the baseline band and the sixth inside
it (none below, none collapsing; class means 66.2 vs.\ 60.9), so the
class-level ordering (placebo $\geq$ gold) is seed-replicated. Remaining
single-seed cells are marked in the text.}
\label{tab:channel}
\end{table}
\footnotetext{The self-report arm's last-6 mean coincides with the baseline's
(both 49.5\%) by arithmetic accident: the two arms' final six validation points
differ individually but share the same sum. Verified against raw logs; not a
transcription error.}

Four confounds separate the auxiliary-loss arm from baseline (prompt format,
predict-block generation, extra update, signal content). The mean-only pair
differs in signal presence, not prompt alone, so it bounds the \emph{joint}
effect of the first two plus the mean-channel signal at ${\approx}0$ at 4B,
assuming no mutual cancellation (the study's only direct prompt-only arm, at
8B, shows a $+7.1$ point-estimate gain with overlapping intervals,
Table~\ref{tab:e01}, so the prompt-side term is not separately measured and the
bound is channel- and scale-local); the shuffled-gold placebo preserves
compute/update-count/mask structure while destroying content--context pairing,
separating the third from the fourth. The placebo verdict landed on the
informative side: shuffled-gold \emph{matches} true gold (76.0 vs.\ 69.3;
descriptively indistinguishable), so exact content--context pairing is \emph{not}
the active ingredient, and the channel gap survives strengthened: on class means
the loss-channel arms beat every reward-channel variant \emph{without requiring
correct labels} (the weakest gold seed, 50.0, lands level with the mean-only
arms' 51.6--52.6). Seed replication
of the true-gold arm (69.3/63.5/50.0 at seeds 0/42/96) sets the effect size,
$\mathbf{+14.2}$ on matched-seed pairing (SE 8.0; $+19.8/+24.4/-1.6$), every seed above
the three-seed baseline mean though the third sits 1.6 below its matched baseline,
with three qualitatively different phase-transition timings (steps ${\sim}70$,
${\sim}100$, ${\sim}25$) all converging to an elevated plateau.

Two further controls complete a graded attribution ladder (full detail in
Appendix~\ref{app:attr-detail}): the preregistered random-vocabulary strong
placebo reaches \textbf{73.4\%}, statistically tied with gold (69.3) and
shuffled gold (76.0); a compute-matched replay control gains nothing (48.9);
and a format-stripped random-token placebo captures a majority of the effect
(61.5). The ladder reads baseline 49.5 ${\approx}$ compute replay 48.9 $<$
random-token 61.5 $<$ schema placebos 73.4--76.0 ${\approx}$ gold 69.3;
compute contributes nothing and content contributes nothing \emph{positive}
at any level (formally, the gold class mean sits below the placebo class mean
at both matched seeds, \S\ref{sec:e01}). Figure~\ref{fig:aux-headline} shows
the headline arm's three-phase dynamics and its zero-interference probe. A preregistered gradient-noise
falsifier (structure-matched Rademacher $\pm\beta$ sign noise) lands at
43.2\% last-6, inside the baseline band: the gain requires \emph{directed}
novel-target gradients, not update noise. Stated at its strongest: on class
means, a format-consistent auxiliary CE pass on \emph{arbitrary} strings
outperforms every reward-channel delivery of a genuinely informative signal;
gains attributed to world modeling in prior work may be substantially
attributable to the auxiliary update itself.

\section{Discussion: Boundaries, Ablations, and Formal Evaluation}
\label{sec:generalization}

Two extensions probe the boundary of the channel result: the same auxiliary-loss recipe
at 8B, and the matrix core replicated on a second environment.

\subsection{At 8B the loss channel turns bistable}
Running the auxiliary-loss recipe unchanged on Qwen3-8B (same $8{\times}4$
configuration, 150 steps) breaks the no-seed-below-baseline pattern (a
formal-caliber pattern at 4B: 68.6/57.9/52.1 all above their matched
42.9/40.7/48.6; at training caliber the third 4B seed already sat 1.6 below
its matched baseline), in both directions.
Seed 0: a three-point early transient (28.1/34.4/25.0, comparable to the
baseline's contemporaneous band rather than above it, so not a honeymoon in the
Table~\ref{tab:collapse} sense) collapses at steps 15--20 into
\emph{27 consecutive zero-success evaluations} (130 steps; the deepest absorbing state
we measure), with the full terminal phenotype: entropy
$0.45{\to}0.03$, per-step \emph{response} length $179{\to}36$ tokens (episode
length does not pin here: this failure class bypasses the length-pinning channel,
consistent with the criterion not covering it), prediction-parse validity and
all-fail-group fraction both pinned at $1.000$. Seed 42: the \emph{same} collapse
onset (three zero evaluations, steps 20--30), then spontaneous escape and a climb
to a last-6 mean of \textbf{72.4\%} (peak 90.6; formal \textbf{85.0\%},
$\mathbf{+40.0}$ over the formal 45.0\% baseline, the study's best gold-signal
endpoint). Same recipe,
data, and budget; 72 points of outcome variance across two seeds. This is not
generic fine-tuning seed variance
\citep{dodge2020finetuning,mosbach2021stability}: two of the three seeds enter
the same early collapse window (steps 15--30) and the third locks later
(${\sim}$step 85, from the lead), yet every locked outcome shows the same
absorbing phenotype, so the outcomes separate into two mechanistically distinct
attractors reached by different routes rather than a diffuse spread (cf.\
scale-emergent instabilities in pretraining, \citealp{wortsman2023smallscale}). Two observations matter for the paper's
claims. First, this collapse occurs \emph{without} the reward channel or z-score
amplification (auxiliary CE pressure under task-gradient starvation), a failure
class the variance-trajectory criterion does not cover. Second, the loss
channel's no-harm record is scale-qualified: at 4B the auxiliary channel
was safe in ten of ten completed ALFWorld runs; at 8B gold full-weight locks two
of three seeds and sends the third to the study's best gold-signal result, a
scale dual of the reward channel's 4B law (there: collapse certain, timing
random; here: occurrence itself random).

Family controls localize the lock (Fig.~\ref{fig:8b-family}; full detail in
Appendix~\ref{app:env-detail}): shuffled-gold
climbs healthy to 78.6 (formal 79.3, level with the escape seed's 85.0), so
the \emph{gain} side of the content-free attribution carries to 8B;
prompt-only stays healthy (47.4), so the prompt alone explains nothing; a
half-weight arm on the locked seed escapes (65.1), so the coefficient
participates causally in basin entry (no tension with
Proposition~\ref{prop:lambda}: $\beta$ never passes through advantage
normalization); and the third gold seed locks from the lead (71.9 $\to$ 0 at
${\sim}$85). Gold full-weight locks two of three seeds; every content-free or
reduced-weight arm stays healthy; under a content-blind null the two healthy
auxiliary-CE arms have probability $(1/3)^2 \approx 0.11$ (0.04 counting
prompt-only), suggestive but six samples remain six
(\S\ref{sec:limitations}).

\begin{figure}[t]
  \begin{minipage}[t]{0.49\linewidth}
    \centering
    \includegraphics[width=\linewidth]{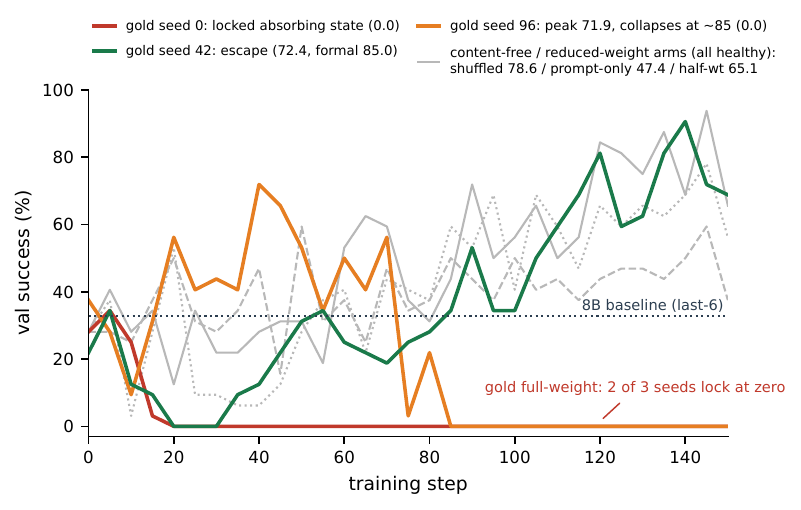}
    \caption{The 8B family. Gold full-weight auxiliary CE is bimodal across
    seeds: two of three seeds lock at literal zero (one after leading the field
    at 71.9\%), one escapes to the study's best gold-signal endpoint (85.0\%).
    Every content-free or reduced-weight arm (shuffled, prompt-only,
    half-weight $\beta{=}0.05$; gray) stays healthy. The risky ingredient at 8B
    is the true-gold signal at full weight, not the auxiliary pass itself.}
    \label{fig:8b-family}
  \end{minipage}\hfill
  \begin{minipage}[t]{0.48\linewidth}
    \centering
    \includegraphics[width=\linewidth]{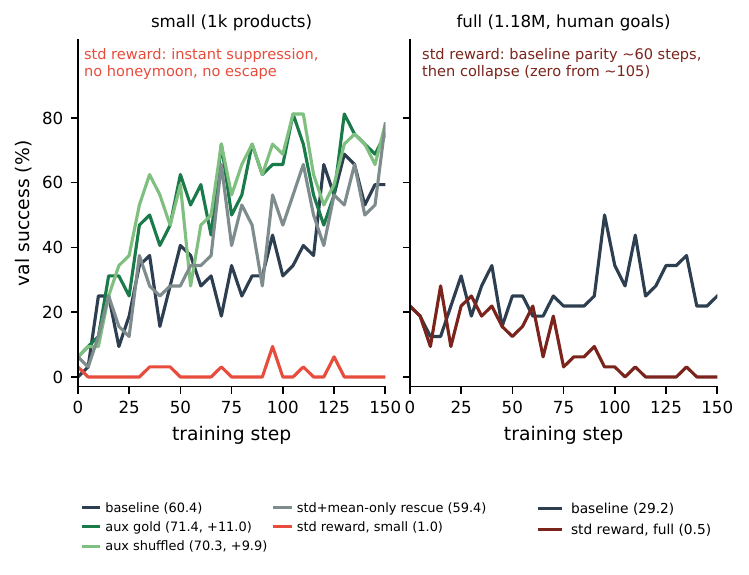}
    \caption{WebShop validation curves. Left: small mode, where the std reward
    arm is suppressed from step~0 (no honeymoon) while the auxiliary-loss arm
    beats the baseline by $+11.0$ training caliber (formal $+3.0$,
    statistically level). Right: full mode, where the std arm holds baseline
    parity for ${\sim}60$ steps, then collapses to a sustained floor,
    reproducing the ALFWorld collapse arc (turn, then sustained zero) with the
    honeymoon replaced by baseline parity, on a comparable timescale.}
    \label{fig:webshop-generalization}
  \end{minipage}
\end{figure}

\subsection{On WebShop the reward channel fails again; the loss-channel gain does not confirm}
\begin{table}[H]
\centering
\small
\begin{tabular}{lcc}
\toprule
arm (WebShop, 4B, last-6) & small (1k products) & full (1.18M, human goals) \\
\midrule
baseline & 60.4\% & 29.2\% \\
std prediction reward & \textbf{1.0\%} & \textbf{0.5\%} \\
auxiliary loss (true gold) & \textbf{71.4\%} ($+11.0$) & --- \\
auxiliary loss (shuffled gold) & \textbf{70.3\%} ($+9.9$) & --- \\
std reward $+$ mean-only rescue & \textbf{59.4\%} & --- \\
\bottomrule
\end{tabular}
\caption{The matrix core on a second environment
(Figure~\ref{fig:webshop-generalization}). Both std reward arms fail, in
mode-dependent forms: the small-mode arm does not leave the floor (max 9.4\%, no
honeymoon, no escape, valid-action ratio $1.000$, so not a parsing artifact); the
full-mode arm tracks its baseline for ${\sim}60$ steps, then collapses (sustained
zero from step ${\sim}105$). The auxiliary-loss arm's training-caliber gain ($+11.0$) compresses to $+3.0$
under the formal 200-goal evaluation (\S\ref{sec:e01} protocol family), statistically
level with its baseline. Success counts episodes the environment marks
\texttt{won} (a perfect terminal matching score); the training reward is the
continuous score itself (Appendix~C), and the 200-goal formal protocol uses the
same success definition.}
\label{tab:webshop}
\end{table}

The reward-channel failure generalizes with regime-dependent form (a
four-form taxonomy: honeymoon-collapse on ALFWorld; collapse-after-parity on
full-mode WebShop, within the 30--135 val-zero range the ALFWorld arms span
with group-size ablations included; instant suppression on small-mode WebShop;
chronic drag on the capped arms), the attribution and the rescue both transfer
(shuffled placebo 70.3 vs.\ gold 71.4, formally 67.0 vs.\ 63.0; mean-only
59.4, a $\mathbf{+58}$-point rescue over the std arm, a cross-normalization comparison
that could only flatter the arm and it still does not beat the baseline), and
the auxiliary gain is seed-split ($+11.0$/$-9.9$ training; $+3.0$/$-14.0$
formal), a single-seed-pair difference, not a confirmed effect size. On
HiddenRule-Gym the same auxiliary recipe is statistically inseparable from the
no-signal arm (17.1 vs.\ 10.7) and below the coverage-matched reward arm: the
loss-channel gain has conditions that one environment success plus a
seed-split second do not license extrapolating away. Estimator controls close
the localization: GiGPO (step-level grouping, group-std kept) collapses on
schedule; PPO/GAE (37.5\%), RLOO (40.6 vs.\ its own 44.8 baseline), and
REINFORCE++ (global-batch whitening, completed after the main freeze: 48.4
vs.\ its own 55.2 baseline, no collapse window) survive the identical recipe;
collapse tracks whether the advantage \emph{divides by a within-group
$\sigma$}, with the signal necessary (\S\ref{sec:rescue}) but not by itself
deciding. Full curves, per-arm paragraphs, and the estimator qualifiers (all
training-caliber, single-seed; the REINFORCE++ signal arm post-dates the
freeze) are in Appendix~\ref{app:env-detail}.

\subsection{Separating coverage, dynamics, and capacity}
\label{sec:hrg}

HRG's exactly computable coverage separates what ALFWorld conflates: the
coverage axis (no-signal floor with gradient starvation; non-covering $\Phi$
erodes belief probes; the coverage sweep pays at mid coverage,
Fig.~\ref{fig:csweep}) and the capacity axis (4B lifts off the 1.7B floor
without cracking rules; training caliber). Full arm-by-arm detail is in
Appendix~\ref{app:env-detail}.

\paragraph{Failure matrix.}
Low coverage $=$ wrong beliefs (B/C); adequate coverage $+$ bad dynamics $=$ collapse
(ALFWorld $+$ std); coverage $+$ clean dynamics $+$ insufficient capacity $=$ partial
conversion far below mastery (arm D at 1.7B: 24.3\% vs.\ 10.7\% no-signal at the
140-game endpoint, a gap partly attributable to the cross-normalization
difference noted above); all three satisfied $=$ parity (rescued arm at 4B ALFWorld). The
gain regime requires switching channels (\S\ref{sec:channel}).
Figure~\ref{fig:csweep} plots the coverage sweep at its formal endpoints.

\begin{figure}[h]
  \centering
  \includegraphics[width=0.55\linewidth]{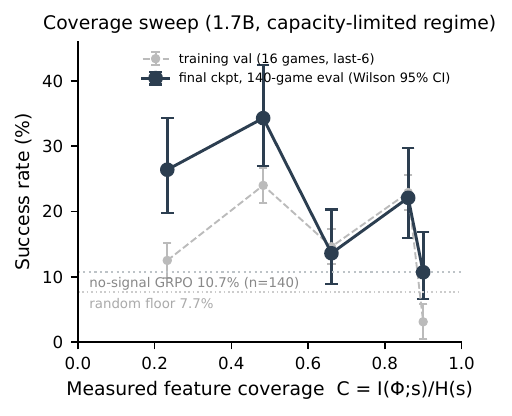}
  \caption{Coverage sweep at measured $C$ (greedy mask ladder), all five rungs at the
  140-game endpoint protocol (Wilson 95\% CIs): 26.4 / \textbf{34.3} / 13.6 / 22.1 /
  10.7\%, against the no-signal GRPO arm's 10.7\% on the same 140 games (dotted line;
  the gray dashed curve is the training-time 16-game series, whose state-sampling
  error reached $+13.9$ points). The endpoints tell a two-regime story. Left: coverage
  buys success; the peak rung's interval [27.0, 42.5] is fully separated from the
  no-signal reference (a 23.6-point gap: interval-separated yet under the
  ${\ge}25$-point categorical bar of \S\ref{sec:e01}, and cross-normalization
  per the caveat below). Right extreme: at full coverage the prediction target itself
  overwhelms 1.7B capacity: prediction accuracy plateaus at 0.70--0.80 far below
  ceiling, terminal entropy drifts to the random-policy level (0.44), and success
  returns to the no-signal level (10.7\%), the coverage benefit fully
  consumed, under mean-only normalization, i.e., independently of the std
  amplification pathology. The middle rung's dip (13.6) overlaps its neighbors'
  intervals and we leave the mid-coverage shape unresolved. All rungs single-seed.
  One design caveat: the no-signal reference is std-normalized while the coverage
  arms are mean-only, so the gaps fold in a normalization difference. On ALFWorld
  the same normalizer change measures $+3.1$ at training caliber but $+13.8$ at
  the formal endpoint (\S\ref{sec:rescue}), so this component cannot be assumed
  small, and the interval separations above license only arm-versus-arm
  differences, not their attribution to coverage alone; a mean-only no-signal arm
  would separate this cleanly and is left to future work.}
  \label{fig:csweep}
\end{figure}

\subsection{Formal endpoint evaluation}
\label{sec:e01}

Last-6 training-caliber means (32-game sampled validation, binomial noise
${\approx}\pm 8.8$pt; 16-game for HiddenRule-Gym, ${\approx}\pm 12.5$pt) are
adequate for shape and large qualitative gaps but not for cross-arm ordering.
We therefore re-evaluate the final
checkpoint of every arm entering a cross-arm ordering claim (the arms of
Table~\ref{tab:e01}, all 11 HiddenRule-Gym arms, and the WebShop arms; the
mechanism battery keeps training-caliber endpoints, \S\ref{sec:limitations}, and
its arm-to-arm statements are deliberately not treated as formal orderings)
under a single protocol: the full 140-game
\texttt{valid\_seen} split and the 134-game \texttt{valid\_unseen} split (padded
to 140 episodes by the fixed-batch loader, six games run twice), same
decoding (sampled, temperature 0.4), same hardware, one pass per arm
(Table~\ref{tab:e01}). All 11 HiddenRule-Gym arms get the same treatment
(140 games on the held-out probe rule families, fixed seed); corrections there
reach $+13.9$ points, twice revising a verdict (arm D, \S\ref{sec:hrg}; the
full-coverage rung, Fig.~\ref{fig:csweep}), and Figure~\ref{fig:csweep} plots the sweep
with its formal intervals. WebShop arms use the analogous 200-goal protocol
(test goals 0--199, one episode each; formal numbers appear in the text and
Appendix~\ref{app:inventory}, training-caliber numbers in Table~\ref{tab:webshop}).

Four things survive formalization, one is revised, and one weakens.
(i)~The channel dichotomy sharpens. From the 44.1\% baseline three-seed mean,
the compute-matched control sits at baseline (48.6, a $+4.5$ point-estimate
difference inside the interval). The auxiliary gain against its
\emph{matched-normalization} baseline is $\mathbf{+15.4}$ (gold three-seed formal mean
59.5; on the matched seed-0 pairing the gap is categorical by the caption's own
bar: 68.6 vs.\ 42.9, $\mathbf{+25.7}$, intervals disjoint), and the content attribution runs entirely within that matched machinery:
under the identical std-plus-auxiliary recipe, the content-free placebos sit at
or above the gold mean (67.9--87.1 vs.\ 59.5; the seed-paired form of this
ordering is given in (ii) below). A separate, cross-normalization
observation: fixing the normalizer alone, with no signal of any kind, reaches
the same level (mean-only no-signal, 57.9), so the auxiliary pass is not the
only road to ${\sim}58$; we flag this convergence as descriptive, not a matched
contrast (the arms differ in normalizer, exactly the comparison this section
otherwise disallows), and the combined cell (mean-only \emph{plus} auxiliary)
is untested. The separation runs along the delivery machinery; the supervision
content contributes nothing positive. Meanwhile no reward-channel variant sits
measurably above its matched-normalization baseline; several point estimates
sit below it (self-report $-5.5$, $\Delta$acc $-12.0$, decoupled $-16.2$
against the 44.1\% mean), though these gaps stay under the caption's categorical
threshold with intervals overlapping the baseline's. The collapsed arms are
confirmed at or statistically indistinguishable from literal zero (upper CI
bounds $\leq$3.9\%). Pairings are stated per comparison throughout this section:
gaps in this paragraph use the three-seed baseline mean (44.1), the seed-0
ladder in (ii) uses the matched seed-0 baseline (42.9). (ii)~The content-free attribution is the study's most consistent formal
pattern, stated within scale: at 4B the shuffled-gold s0 placebo (81.4) and
random-vocabulary arm (87.1) sit above every true-gold seed (68.6/57.9/52.1),
while the shuffled s42 and random-token endpoints (67.9) do not; the surviving
statement is a descriptive, seed-paired class ordering, not a categorical
claim (full ladder restatement and calibration notes in
Appendix~\ref{app:e01-detail}). (iii)~The mean-only rescue and its no-signal control stay within overlapping
intervals (52.9 vs.\ 57.9, a 5.0-point gap at point estimate, the same
evidentiary grade as the self-report gap), preserving the net-zero attribution
of the cleaned reward channel.
(iv)~The compute control stays at baseline (48.6; $+4.5$ at point estimate,
inside the interval, the same evidentiary grade as the negative point-estimate
gaps above). The revision: the 8B baseline evaluates at \textbf{45.0\%} [37.0, 53.3], inside the
4B baseline seed band, so the training-time anomaly (32.8\% $<$ 4B) was
substantially a 32-game-window artifact, and scale-ordering caveats tied to it are
lifted. Family-arm gains restate on formal numbers against the 45.0 baseline (gold escape
seed: $+40.0$; shuffled: $+34.3$; both intervals disjoint from the baseline's). \paragraph{Out of distribution and the largest revision.}
On the 134-game unseen split every headline structure survives and every
magnitude compresses (collapsed arms literal zero, upper bounds 2.7\%;
auxiliary gain $+15.4 \to +7.9$; four near-baseline gaps flip sign within
noise); the largest downward revision is self-report ($49.5 \to 38.6$), the
reason single-seed rows carry a cross-seed caveat. Full out-of-distribution
and calibration prose is in Appendix~\ref{app:e01-detail}.

\begin{table}[H]
\centering
\scriptsize
\begin{tabular}{llcccc}
\toprule
arm & channel / role & seen (140) & [95\% CI] & unseen (134${\to}$140) & last-6 \\
\midrule
\multicolumn{6}{l}{\emph{1.7B}} \\
baseline & --- & 17.9 & [12.4, 25.1] & 18.6 & 20.9 \\
std prediction reward & reward & 0.0 & [0.0, 2.7] & 0.0 & 0.0 \\
\midrule
\multicolumn{6}{l}{\emph{4B}} \\
baseline (s0) & --- & 42.9 & [35.0, 51.2] & 42.1 & 49.5 \\
baseline (s42) & --- & 40.7 & [32.9, 49.0] & 37.1 & 39.1 \\
baseline (s96) & --- & 48.6 & [40.5, 56.8] & 46.4 & 51.6 \\
aux loss, gold (s0) & loss & 68.6 & [60.5, 75.7] & 65.0 & 69.3 \\
aux loss, gold (s42) & loss & 57.9 & [49.6, 65.8] & 45.0 & 63.5 \\
aux loss, gold (s96) & loss & 52.1 & [43.9, 60.2] & 39.3 & 50.0 \\
aux loss, shuffled gold (s0) & loss placebo & \textbf{81.4} & [74.1, 87.0] & \textbf{73.6} & 76.0 \\
aux loss, shuffled (s42) & loss placebo & 67.9 & [59.8, 75.1] & 66.4 & 60.4 \\
aux loss, random vocab (s0) & loss placebo & \textbf{87.1} & [80.5, 91.7] & 71.4 & 73.4 \\
aux loss, random tokens (s0) & loss placebo & 67.9 & [59.8, 75.1] & 57.9 & 61.5 \\
2-epoch compute control & control & 48.6 & [40.5, 56.8] & 39.3 & 48.9 \\
std prediction reward & reward & 0.7 & [0.1, 3.9] & 0.0 & 1.0 \\
mean-only rescue & reward & 52.9 & [44.7, 61.0] & 40.7 & 51.6 \\
mean-only, no signal & control & 57.9 & [49.6, 65.8] & 44.3 & 52.6 \\
all-fail-group filter & reward & 40.7 & [32.9, 49.0] & 35.7 & 39.6 \\
decoupled (capped) reward & reward & 27.9 & [21.1, 35.8] & 30.0 & 31.2 \\
$\Delta$-accuracy reward & reward & 32.1 & [24.9, 40.2] & 29.3 & 28.1 \\
anchor-QA (recall) reward & reward & 44.3 & [36.3, 52.6] & 33.6 & 44.3 \\
self-report reward & reward & 38.6 & [30.9, 46.9] & 39.3 & 49.5 \\
\midrule
\multicolumn{6}{l}{\emph{8B}} \\
baseline & --- & 45.0 & [37.0, 53.3] & 39.3 & 32.8 \\
aux loss, gold (s0, locked) & loss & 0.0 & [0.0, 2.7] & 0.0 & 0.0 \\
aux loss, gold (s42, escaped) & loss & \textbf{85.0} & [78.2, 90.0] & 72.1 & 72.4 \\
aux loss, shuffled & loss & 79.3 & [71.9, 85.2] & 66.4 & 78.6 \\
prompt-only & --- & 52.1 & [43.9, 60.2] & 44.3 & 47.4 \\
std prediction reward & reward & 0.0 & [0.0, 2.7] & 0.0 & 0.0 \\
\midrule
\multicolumn{6}{l}{\emph{4B, late-finishing}} \\
std prediction reward (s42) & reward & 0.0 & [0.0, 2.7] & 0.0 & 0.0 \\
decoupled normalization (s42) & reward & 32.1 & [24.9, 40.2] & 26.4 & 42.2 \\
\bottomrule
\end{tabular}
\caption{Unified formal evaluation (E01): final checkpoints, 140-game
\texttt{valid\_seen}, sampled decoding at temperature 0.4, Wilson 95\% intervals
(computed by the evaluation driver from the logged three-decimal success rate at
$n{=}140$; recomputing from raw counts can shift a bound's last digit by
0.1pt).
Intervals are within-arm binomial only; they do \emph{not} include seed variance,
which the multi-seed arms show spans 4--17 points peak-to-peak among the healthy
4B arms (and 85 points in the bistable 8B gold family); single-seed rows
(including both table-topping placebos) carry that additional cross-seed uncertainty
(\S\ref{sec:limitations}). With 29 arms, uncorrected pairwise comparisons invite
multiplicity concerns; the categorical statements in the text rest only on gaps of
${\ge}25$ points with non-overlapping intervals, which survive any standard
correction, and fine-grained rankings among adjacent arms are deliberately not
interpreted. Training-time last-6 means shown for continuity with the
curves in the text; prose retains last-6 numbers where it describes training dynamics.
The unseen column is the full 134-game \texttt{valid\_unseen} split evaluated
under the same fixed-batch protocol, which pads it to 140 episodes (six games
run twice); percentages and intervals are over the 140 episodes, so the upper
bound at zero is 2.7\%. Formal numbers
for the estimator controls, the 8B gold s96 and half-weight arms, the $\Delta$acc
second seed, and the remaining placebo seed
replications (shuffled s96; random-vocabulary s42/s96) were pending at the data
freeze; other late-finishing family arms are included above. The anchor-QA arm's
training last-6 and formal endpoint coincide at 44.3 by arithmetic accident; the
underlying counts differ.}
\label{tab:e01}
\end{table}

\subsection{Practical guidance}
\label{sec:guidance}

In brief (full prescriptions with numbers in Appendix~\ref{app:guidance-ext}):
(1)~do not mix difference-form self-prediction shaping into GRPO's
std-normalized reward when within-group task-return variance is small relative
to the shaping term; (2)~if reward-channel shaping is unavoidable, use
mean-only normalization (decoupling under-performs; annealing rescues only
inside the honeymoon window); (3)~prefer signals whose within-group variance
decays, and verify the decay empirically ($\Delta$acc is the cautionary test);
(4)~for the gain, add an auxiliary CE pass \emph{as a regularizer}: any
well-formed target works, do not pay for gold supervision, and treat scale as
hazardous (at 8B run a seed pair and consider halving the weight); (5)~monitor
the joint precursor, never entropy alone; (6)~read feature-set difficulty off
the prediction-saturation form before spending compute.

\subsection{Limitations}
\label{sec:limitations}

\paragraph{Seeds and statistics.}
The headline auxiliary-loss arm, its two leading placebos, and the ALFWorld
baseline carry three seeds; the WebShop pair, decoupled, and $\Delta$acc carry
two; most mechanism arms are single-seed, so cross-arm gaps inherit partially
quantified seed variance (healthy-4B spread 4--17 points peak-to-peak; 85 in
the bistable 8B gold family, whose six occupancy samples support a structural
description, not a rate estimate). Every arm entering a cross-arm ordering
claim has a formal 140-game endpoint and the headline gaps survive it; the
mechanism battery and the caption's pending cells keep training-caliber
endpoints only, and the out-of-distribution pass (ALFWorld only) preserves
every headline structure while magnitudes compress
(\S\ref{sec:e01}).

\paragraph{Scope.}
Three environments carry results (ALFWorld, WebShop, HiddenRule-Gym); a fourth,
ScienceWorld, appears in the arm inventory (Appendix~\ref{app:inventory}) as in-flight arms only
(one running, one queued) and contributes no interpreted result. The WebShop shuffled placebo (70.3) and the
mean-only rescue of the instant-suppression form (59.4) are completed, so
attribution and prescription carry cross-environment evidence. Collapse-timing
observations rest on three scale points, and HRG conclusions are
capacity-limited at 1.7B (the 4B capacity gate shows lift-off without rule
learning). The 8B baseline anomaly is resolved by the unified evaluation (the
training-time 32.8\% evaluates at 45.0\% over 140 games, inside the 4B baseline
seed band), so scale-ordering claims are stated on formal numbers only
(\S\ref{sec:e01}). All results are on a single base-model family (Qwen3, three
sizes) with one agent framework and one prediction-signal design (the
inventory's Qwen2.5-1.5B row is an early full-budget pipeline anchor, not a
comparison arm); the mechanism is architecture-agnostic algebra, but its
empirical constants (collapse timing, honeymoon peaks, gain magnitudes) should
be read as Qwen3-specific until replicated elsewhere. Estimator scope is four
cells wide: GiGPO reproduces the collapse, PPO/GAE and RLOO survive the
identical recipe, and collapse aligns with the within-group $\sigma$ division
across all four (\S\ref{sec:generalization}); the last cell, global-batch
whitening (REINFORCE++), completed after the main data freeze and survives the
identical recipe (48.4\% last-6 against its own 55.2\% baseline arm, no
collapse window, no terminal triple), localizing the warning to the
within-group $\sigma$ division rather than std normalization per se
(training-caliber, single-seed, post-freeze: reported as a boundary note, not
entered into the formal ordering tables).

\paragraph{Resolved since the previous version.}
The small-group-$\hat\sigma$ collinearity concern and the single-environment
concern were both closed by preregistered ablations (group sizes $8/16$;
WebShop replication); detail in Appendix~\ref{app:env-detail}.

\section{Conclusion}
\label{sec:conclusion}

Dense prediction rewards are a natural and popular remedy for sparse-reward LLM
agents, and under group-normalized RL they destroy the policy: at every scale,
group size, and fixed shaping coefficient we tested, std normalization converts
a bounded, telescoping, verifier-graded signal into full-scale pressure toward
predictability, and the optimizer builds a dark room. The mechanism is a
one-line algebraic fact with three-point empirical confirmation; the danger is
governed by the signal's within-group variance trajectory; the only
interventions that work are structural and preventive; and when the same signal
does help, through an auxiliary loss, placebo controls trace the gain to the
update rather than the world model. Both halves come with boundaries. The
reward-channel failure and the content-free attribution replicate on a second
environment, where task difficulty selects the failure's form. At 8B the
auxiliary channel becomes seed-determined: gold full-weight locks two of three
seeds and sends the third to the study's best gold-signal result, while every
content-free or reduced-weight arm stays healthy. Scale moves the randomness
from collapse timing to collapse occurrence, and moves content from inert to
risk-implicated: the gain is content-free everywhere, but the one 8B
configuration that locks is the true-gold signal at full weight (a structural
reading from a six-run family, \S\ref{sec:limitations}). We report the
mechanism, the criterion, the early-warning signature, and the ledger of
preregistered predictions so the next design iteration of dense supervision in
agent RL can be checked against them.

\subsubsection*{Reproducibility Statement}
All experiments use a single aligned configuration across scales (train batch 8 $\times$
group 4, validation 32 episodes at temperature 0.4 (16 for HiddenRule-Gym at
1.7B; the 4B HRG capacity arm used 32), seed 0 unless stated;
Appendix~\ref{app:config}). Training-dynamics numbers are last-6 validation means from
training-time evaluation; formal endpoint numbers come from the unified 140-game
evaluation of every arm entering a cross-arm ordering claim (\S\ref{sec:e01},
Table~\ref{tab:e01}); the mechanism battery keeps training-caliber endpoints
(\S\ref{sec:limitations}). Seed replication (seeds 42/96) covers the headline
auxiliary-loss arm, its two leading placebo controls, the ALFWorld baseline, and
the 8B gold arm, with second seeds also on the decoupled, $\Delta$acc, and
WebShop auxiliary arms; the random-token placebo and the 8B shuffled,
prompt-only, and
half-weight arms are single-seed (marked in Table~\ref{tab:e01} and
Appendix~\ref{app:inventory});
Appendix~\ref{app:inventory} inventories all 74 training arms of the study with their
endpoints; both evaluation splits
(140-game seen, 134-game unseen) are complete for every arm in
Table~\ref{tab:e01}, and formal evaluation of late-finishing arms was pending at
freeze time.
The agent framework builds on the open-source verl-agent stack. Our code is
available at \url{https://github.com/RobertWangWang/verl-agent} (a fork of
\url{https://github.com/langfengQ/verl-agent}): the reward pipeline, the auxiliary-CE
channel with all placebo modes, HiddenRule-Gym, the WebShop and ScienceWorld service
harnesses, run recipes, unit tests, and the preregistration digest file whose git
history timestamps every registered prediction.

\bibliography{refs}

\begin{thebibliography}{66}
\providecommand{\natexlab}[1]{#1}
\providecommand{\url}[1]{\texttt{#1}}
\expandafter\ifx\csname urlstyle\endcsname\relax
  \providecommand{\doi}[1]{doi: #1}\else
  \providecommand{\doi}{doi: \begingroup \urlstyle{rm}\Url}\fi

\bibitem[Baltieri \& Buckley(2019)Baltieri and Buckley]{baltieri2019darkroom}
Manuel Baltieri and Christopher~L. Buckley.
\newblock The dark room problem in predictive processing and active inference,
  a legacy of cognitivism?
\newblock In \emph{Artificial Life Conference Proceedings}, 2019.

\bibitem[Bay \& Yearick(2026)Bay and Yearick]{bay2026identity}
Yong~Yi Bay and Kathleen~A. Yearick.
\newblock {GRPO}, {Dr.GRPO}, and {DAPO} are three operations on one number: The
  group-standard-deviation identity.
\newblock \emph{arXiv preprint arXiv:2607.00152}, 2026.

\bibitem[Behboudian et~al.(2021)Behboudian, Satsangi, Taylor, Harutyunyan, and
  Bowling]{behboudian2021pies}
Paniz Behboudian, Yash Satsangi, Matthew~E. Taylor, Anna Harutyunyan, and
  Michael Bowling.
\newblock Policy invariant explicit shaping: an efficient alternative to reward
  shaping.
\newblock \emph{Neural Computing and Applications}, 2021.

\bibitem[Bereket \& Leskovec(2025)Bereket and
  Leskovec]{bereket2025uncalibrated}
Michael Bereket and Jure Leskovec.
\newblock Uncalibrated reasoning: {GRPO} induces overconfidence for stochastic
  outcomes.
\newblock \emph{arXiv preprint arXiv:2508.11800}, 2025.

\bibitem[Burda et~al.(2019)Burda, Edwards, Pathak, Storkey, Darrell, and
  Efros]{burda2019curiosity}
Yuri Burda, Harrison Edwards, Deepak Pathak, Amos Storkey, Trevor Darrell, and
  Alexei~A. Efros.
\newblock Large-scale study of curiosity-driven learning.
\newblock In \emph{ICLR}, 2019.

\bibitem[Champion et~al.(2024)Champion, Grze{\'s}, Bonheme, and
  Bowman]{champion2024deconstructing}
Th{\'e}ophile Champion, Marek Grze{\'s}, Lisa Bonheme, and Howard Bowman.
\newblock Deconstructing deep active inference: a contrarian information
  gatherer.
\newblock \emph{Neural Computation}, 2024.

\bibitem[Cui et~al.(2025{\natexlab{a}})]{cui2025entropy}
Ganqu Cui et~al.
\newblock The entropy mechanism of reinforcement learning for reasoning
  language models.
\newblock \emph{arXiv preprint arXiv:2505.22617}, 2025{\natexlab{a}}.

\bibitem[Cui et~al.(2025{\natexlab{b}})]{cui2025prime}
Ganqu Cui et~al.
\newblock Process reinforcement through implicit rewards.
\newblock \emph{arXiv preprint arXiv:2502.01456}, 2025{\natexlab{b}}.

\bibitem[{DeepSeek-AI}(2025)]{guo2025deepseekr1}
{DeepSeek-AI}.
\newblock {DeepSeek-R1}: Incentivizing reasoning capability in {LLMs} via
  reinforcement learning.
\newblock \emph{arXiv preprint arXiv:2501.12948}, 2025.

\bibitem[Devlin \& Kudenko(2012)Devlin and Kudenko]{devlin2012dynamic}
Sam Devlin and Daniel Kudenko.
\newblock Dynamic potential-based reward shaping.
\newblock In \emph{AAMAS}, 2012.

\bibitem[Dodge et~al.(2020)Dodge, Ilharco, Schwartz, Farhadi, Hajishirzi, and
  Smith]{dodge2020finetuning}
Jesse Dodge, Gabriel Ilharco, Roy Schwartz, Ali Farhadi, Hannaneh Hajishirzi,
  and Noah Smith.
\newblock Fine-tuning pretrained language models: Weight initializations, data
  orders, and early stopping.
\newblock \emph{arXiv preprint arXiv:2002.06305}, 2020.

\bibitem[Feng et~al.(2025)Feng, Xue, Liu, and An]{feng2025gigpo}
Lang Feng, Zhenghai Xue, Tingcong Liu, and Bo~An.
\newblock Group-in-group policy optimization for {LLM} agent training.
\newblock In \emph{NeurIPS}, 2025.

\bibitem[Forbes et~al.(2024)Forbes, Villalobos-Arias, Xu, Potts, Jhala, and
  Roberts]{forbes2024pbim}
Grant~C. Forbes, Leonardo Villalobos-Arias, Jianxun Xu, Colin~M. Potts, Arnav
  Jhala, and David~L. Roberts.
\newblock Potential-based intrinsic motivation: Preserving optimality with
  complex, non-{Markovian} shaping rewards.
\newblock In \emph{arXiv preprint arXiv:2402.07411}, 2024.

\bibitem[Friston et~al.(2012)Friston, Thornton, and Clark]{friston2012darkroom}
Karl Friston, Christopher Thornton, and Andy Clark.
\newblock Free-energy minimization and the dark-room problem.
\newblock \emph{Frontiers in Psychology}, 3:\penalty0 130, 2012.

\bibitem[Fu et~al.(2025)]{fu2025par}
Jiayi Fu et~al.
\newblock Reward shaping to mitigate reward hacking in {RLHF}.
\newblock \emph{arXiv preprint arXiv:2502.18770}, 2025.

\bibitem[Ge et~al.(2026)]{ge2026curvature}
Cheng Ge et~al.
\newblock Why {GRPO} needs normalization: A local-curvature perspective on
  adaptive gradients.
\newblock \emph{arXiv preprint}, 2026.

\bibitem[Harutyunyan et~al.(2015)Harutyunyan, Devlin, Vrancx, and
  Now{\'e}]{harutyunyan2015dpba}
Anna Harutyunyan, Sam Devlin, Peter Vrancx, and Ann Now{\'e}.
\newblock Expressing arbitrary reward functions as potential-based advice.
\newblock In \emph{AAAI}, 2015.

\bibitem[He et~al.(2026)He, Feng, Wei, Cheng, Feng, and An]{he2026hgpo}
Shuo He, Lang Feng, Qi~Wei, Xin Cheng, Lei Feng, and Bo~An.
\newblock Hierarchy-of-groups policy optimization for long-horizon agentic
  tasks.
\newblock In \emph{International Conference on Learning Representations
  (ICLR)}, 2026.

\bibitem[Hou et~al.(2025)]{hou2025lpm}
Zhibo Hou et~al.
\newblock Beyond noisy-{TVs}: Noise-robust exploration via learning progress
  monitoring.
\newblock \emph{arXiv preprint}, 2025.

\bibitem[Hu et~al.(2025)]{hu2025reinforcepp}
Jian Hu et~al.
\newblock {REINFORCE++}: Stabilizing critic-free policy optimization with
  global advantage normalization.
\newblock \emph{arXiv preprint arXiv:2501.03262}, 2025.

\bibitem[Ichihara et~al.(2025)]{mogrpo2025}
Yuki Ichihara et~al.
\newblock {MO-GRPO}: Mitigating reward hacking of group relative policy
  optimization on multi-objective problems.
\newblock \emph{arXiv preprint arXiv:2509.22047}, 2025.

\bibitem[Jaderberg et~al.(2017)Jaderberg, Mnih, Czarnecki, Schaul, Leibo,
  Silver, and Kavukcuoglu]{jaderberg2017unreal}
Max Jaderberg, Volodymyr Mnih, Wojciech~Marian Czarnecki, Tom Schaul, Joel~Z.
  Leibo, David Silver, and Koray Kavukcuoglu.
\newblock Reinforcement learning with unsupervised auxiliary tasks.
\newblock In \emph{ICLR}, 2017.

\bibitem[Jiang et~al.(2023)Jiang, Chen, Pan, Wang, Dapeng, Jiang, and
  Long]{jiang2023forkmerge}
Junguang Jiang, Baixu Chen, Junwei Pan, Ximei Wang, Liu Dapeng, Jie Jiang, and
  Mingsheng Long.
\newblock {ForkMerge}: Mitigating negative transfer in auxiliary-task learning.
\newblock In \emph{Advances in Neural Information Processing Systems}, 2023.

\bibitem[Leng et~al.(2025)]{leng2025mmr1}
Sicong Leng et~al.
\newblock {MMR1}: Enhancing multimodal reasoning with variance-aware sampling
  and open resources.
\newblock \emph{arXiv preprint}, 2025.

\bibitem[Liu et~al.(2026)]{liu2026gdpo}
Shih-Yang Liu et~al.
\newblock {GDPO}: Group reward-decoupled normalization policy optimization for
  multi-reward {RL} optimization.
\newblock \emph{arXiv preprint arXiv:2601.05242}, 2026.

\bibitem[Liu et~al.(2025)Liu, Chen, Li, Qi, Pang, Du, Lee, and
  Lin]{liu2025drgrpo}
Zichen Liu, Changyu Chen, Wenjun Li, Penghui Qi, Tianyu Pang, Chao Du, Wee~Sun
  Lee, and Min Lin.
\newblock Understanding {R1-Zero}-like training: A critical perspective.
\newblock \emph{arXiv preprint arXiv:2503.20783}, 2025.

\bibitem[Lu et~al.(2026)]{lu2026paw}
Ning Lu et~al.
\newblock Policy and world modeling co-training for language agents.
\newblock \emph{arXiv preprint arXiv:2606.02388}, 2026.

\bibitem[Mavor-Parker et~al.(2022)Mavor-Parker, Young, Barry, and
  Griffin]{mavorparker2022ama}
Augustine~N. Mavor-Parker, Kimberly~A. Young, Caswell Barry, and Lewis~D.
  Griffin.
\newblock How to stay curious while avoiding noisy {TVs} using aleatoric
  uncertainty estimation.
\newblock In \emph{ICML}, 2022.

\bibitem[Min et~al.(2022)Min, Lyu, Holtzman, Artetxe, Lewis, Hajishirzi, and
  Zettlemoyer]{min2022rethinking}
Sewon Min, Xinxi Lyu, Ari Holtzman, Mikel Artetxe, Mike Lewis, Hannaneh
  Hajishirzi, and Luke Zettlemoyer.
\newblock Rethinking the role of demonstrations: What makes in-context learning
  work?
\newblock \emph{EMNLP}, 2022.

\bibitem[Mosbach et~al.(2021)Mosbach, Andriushchenko, and
  Klakow]{mosbach2021stability}
Marius Mosbach, Maksym Andriushchenko, and Dietrich Klakow.
\newblock On the stability of fine-tuning {BERT}: Misconceptions, explanations,
  and strong baselines.
\newblock In \emph{International Conference on Learning Representations}, 2021.

\bibitem[Mroueh(2025)]{mroueh2025grpo}
Youssef Mroueh.
\newblock Reinforcement learning with verifiable rewards: {GRPO}'s effective
  loss, dynamics, and success amplification.
\newblock \emph{arXiv preprint arXiv:2503.06639}, 2025.

\bibitem[Ng et~al.(1999)Ng, Harada, and Russell]{ng1999policy}
Andrew~Y. Ng, Daishi Harada, and Stuart Russell.
\newblock Policy invariance under reward transformations: Theory and
  application to reward shaping.
\newblock In \emph{ICML}, 1999.

\bibitem[Pan et~al.(2022)Pan, Bhatia, and Steinhardt]{pan2022misspecification}
Alexander Pan, Kush Bhatia, and Jacob Steinhardt.
\newblock The effects of reward misspecification: Mapping and mitigating
  misaligned models.
\newblock In \emph{ICLR}, 2022.

\bibitem[Park et~al.(2025{\natexlab{a}})]{park2025cliplow}
Jaesung Park et~al.
\newblock Clip-low increases entropy and clip-high decreases entropy in
  reinforcement learning of large language models.
\newblock \emph{arXiv preprint}, 2025{\natexlab{a}}.

\bibitem[Park et~al.(2025{\natexlab{b}})Park, Yang, Lee, Chang, and
  Zhang]{park2025s2d}
Junseok Park, Hyeonseo Yang, Min~Whoo Lee, Won-Seok Chang, and Byoung-Tak
  Zhang.
\newblock From sparse to dense: Toddler-inspired reward transition in
  goal-oriented reinforcement learning.
\newblock \emph{IEEE Transactions on Cognitive and Developmental Systems},
  2025{\natexlab{b}}.

\bibitem[Pathak et~al.(2017)Pathak, Agrawal, Efros, and
  Darrell]{pathak2017curiosity}
Deepak Pathak, Pulkit Agrawal, Alexei~A. Efros, and Trevor Darrell.
\newblock Curiosity-driven exploration by self-supervised prediction.
\newblock In \emph{ICML}, 2017.

\bibitem[Seth et~al.(2020)Seth, Millidge, Buckley, and
  Tschantz]{seth2020curious}
Anil~K. Seth, Beren Millidge, Christopher~L. Buckley, and Alexander Tschantz.
\newblock Curious inferences: Reply to {Sun} and {Firestone} on the dark room
  problem.
\newblock \emph{Trends in Cognitive Sciences}, 2020.

\bibitem[Setlur et~al.(2025)Setlur, Nagpal, Fisch, Geng, Eisenstein, Agarwal,
  Agarwal, Berant, and Kumar]{setlur2024pav}
Amrith Setlur, Chirag Nagpal, Adam Fisch, Xinyang Geng, Jacob Eisenstein,
  Rishabh Agarwal, Alekh Agarwal, Jonathan Berant, and Aviral Kumar.
\newblock Rewarding progress: Scaling automated process verifiers for {LLM}
  reasoning.
\newblock In \emph{ICLR}, 2025.

\bibitem[Shao et~al.(2024)Shao, Wang, Zhu, Xu, Song, Bi, Zhang, Zhang, Li, Wu,
  and Guo]{shao2024grpo}
Zhihong Shao, Peiyi Wang, Qihao Zhu, Runxin Xu, Junxiao Song, Xiao Bi, Haowei
  Zhang, Mingchuan Zhang, Y.~K. Li, Y.~Wu, and Daya Guo.
\newblock {DeepSeekMath}: Pushing the limits of mathematical reasoning in open
  language models.
\newblock \emph{arXiv preprint arXiv:2402.03300}, 2024.

\bibitem[Shridhar et~al.(2021)Shridhar, Yuan, C{\^o}t{\'e}, Bisk, Trischler,
  and Hausknecht]{shridhar2021alfworld}
Mohit Shridhar, Xingdi Yuan, Marc-Alexandre C{\^o}t{\'e}, Yonatan Bisk, Adam
  Trischler, and Matthew Hausknecht.
\newblock {ALFWorld}: Aligning text and embodied environments for interactive
  learning.
\newblock In \emph{International Conference on Learning Representations
  (ICLR)}, 2021.

\bibitem[Shrivastava et~al.(2026)Shrivastava, Kauffmann, Awadallah, and
  Papailiopoulos]{shrivastava2026echo}
Vaishnavi Shrivastava, Piero Kauffmann, Ahmed Awadallah, and Dimitris
  Papailiopoulos.
\newblock {ECHO}: Terminal agents learn world models for free.
\newblock \emph{arXiv preprint arXiv:2605.24517}, 2026.

\bibitem[Tan et~al.(2026)]{tan2026hcapo}
Huihan Tan et~al.
\newblock Hindsight credit assignment for long-horizon {LLM} agents.
\newblock \emph{arXiv preprint}, 2026.

\bibitem[Tao et~al.(2025)]{tao2025hero}
Leitian Tao et~al.
\newblock Hybrid reinforcement: When reward is sparse, it's better to be dense.
\newblock \emph{arXiv preprint}, 2025.

\bibitem[Voelcker et~al.(2024)Voelcker, Kastner, Gilitschenski, and
  Farahmand]{voelcker2024selfprediction}
Claas Voelcker, Tyler Kastner, Igor Gilitschenski, and Amir-massoud Farahmand.
\newblock When does self-prediction help? {Understanding} auxiliary tasks in
  reinforcement learning.
\newblock \emph{arXiv preprint arXiv:2406.17718}, 2024.

\bibitem[Wang et~al.(2025{\natexlab{a}})]{wang2025igpo}
Guoqing Wang et~al.
\newblock Information gain-based policy optimization: A simple and effective
  approach for multi-turn {LLM} agents.
\newblock \emph{arXiv preprint}, 2025{\natexlab{a}}.

\bibitem[Wang et~al.(2025{\natexlab{b}})]{wang2025vagen}
Kangrui Wang et~al.
\newblock {VAGEN}: Reinforcing world model reasoning for multi-turn {VLM}
  agents.
\newblock \emph{arXiv preprint}, 2025{\natexlab{b}}.

\bibitem[Wang et~al.(2025{\natexlab{c}})]{wang2025practitioner}
Ruiyi Wang et~al.
\newblock A practitioner's guide to multi-turn agentic reinforcement learning.
\newblock \emph{arXiv preprint}, 2025{\natexlab{c}}.

\bibitem[Wang et~al.(2026{\natexlab{a}})]{wang2026hackingsurvey}
Xiaohua Wang et~al.
\newblock Reward hacking in the era of large models: Mechanisms, emergent
  misalignment, challenges.
\newblock \emph{arXiv preprint}, 2026{\natexlab{a}}.

\bibitem[Wang et~al.(2025{\natexlab{d}})Wang, Li, Zang, Zhou, Bu, Wang, Lu,
  Jin, and Wang]{prefgrpo2025}
Yibin Wang, Zhimin Li, Yuhang Zang, Yujie Zhou, Jiazi Bu, Chunyu Wang, Qinglin
  Lu, Cheng Jin, and Jiaqi Wang.
\newblock {Pref-GRPO}: Pairwise preference reward-based {GRPO} for stable
  text-to-image reinforcement learning.
\newblock \emph{arXiv preprint arXiv:2508.20751}, 2025{\natexlab{d}}.

\bibitem[Wang et~al.(2025{\natexlab{e}})]{wang2025ragen}
Zihan Wang et~al.
\newblock {RAGEN}: Understanding self-evolution in {LLM} agents via multi-turn
  reinforcement learning.
\newblock \emph{arXiv preprint arXiv:2504.20073}, 2025{\natexlab{e}}.

\bibitem[Wang et~al.(2026{\natexlab{b}})]{wang2026beacon}
Zixuan Wang et~al.
\newblock Milestone-guided policy learning for long-horizon language agents.
\newblock \emph{arXiv preprint}, 2026{\natexlab{b}}.

\bibitem[Wortsman et~al.(2023)Wortsman, Liu, Xiao, Everett, Alemi,
  et~al.]{wortsman2023smallscale}
Mitchell Wortsman, Peter~J Liu, Lechao Xiao, Katie Everett, Alexander Alemi,
  et~al.
\newblock Small-scale proxies for large-scale transformer training
  instabilities.
\newblock \emph{arXiv preprint arXiv:2309.14322}, 2023.

\bibitem[Wu et~al.(2025)]{wu2025qae}
Junkang Wu et~al.
\newblock Quantile advantage estimation: Stabilizing {RLVR} for {LLM}
  reasoning.
\newblock \emph{arXiv preprint}, 2025.

\bibitem[Xi et~al.(2026)]{xi2025agentprm}
Zhiheng Xi et~al.
\newblock {AgentPRM}: Process reward models for {LLM} agents via step-wise
  promise and progress.
\newblock In \emph{Proceedings of the ACM Web Conference}, 2026.

\bibitem[Xiao et~al.(2025)]{xiao2025bnpo}
Changyi Xiao et~al.
\newblock {BNPO}: Beta normalization policy optimization.
\newblock \emph{arXiv preprint arXiv:2506.02864}, 2025.

\bibitem[Xu et~al.(2026)]{xu2026scpo}
Peng Xu et~al.
\newblock Semantic consistency policy optimization for reinforcement learning
  of {LLM} agents.
\newblock \emph{arXiv preprint}, 2026.

\bibitem[Xu \& Ding(2025)Xu and Ding]{xu2025spo}
Zhongwen Xu and Zihan Ding.
\newblock Single-stream policy optimization.
\newblock \emph{arXiv preprint}, 2025.

\bibitem[Yang et~al.(2026)]{yang2026biased}
Feng Yang et~al.
\newblock Your group-relative advantage is biased.
\newblock \emph{arXiv preprint}, 2026.

\bibitem[Yao et~al.(2022)Yao, Chen, Yang, and Narasimhan]{yao2022webshop}
Shunyu Yao, Howard Chen, John Yang, and Karthik Narasimhan.
\newblock {WebShop}: Towards scalable real-world web interaction with grounded
  language agents.
\newblock In \emph{Advances in Neural Information Processing Systems
  (NeurIPS)}, 2022.

\bibitem[Yu et~al.(2025)Yu, Zhang, Zhu, Yuan, Zuo, Yue, Fan, Liu, Liu, Liu,
  et~al.]{yu2025dapo}
Qiying Yu, Zheng Zhang, Ruofei Zhu, Yufeng Yuan, Xiaochen Zuo, Yu~Yue, Tiantian
  Fan, Gaohong Liu, Lingjun Liu, Xin Liu, et~al.
\newblock Dapo: An open-source llm reinforcement learning system at scale.
\newblock \emph{arXiv preprint arXiv:2503.14476}, 2025.

\bibitem[Yu et~al.(2026)]{yu2026rwml}
Xiao Yu et~al.
\newblock Reinforcement world model learning for {LLM}-based agents.
\newblock \emph{arXiv preprint arXiv:2602.05842}, 2026.

\bibitem[Zhang et~al.(2025{\natexlab{a}})Zhang, Wu, Zhu, et~al.]{zhang2025scaf}
Xichen Zhang, Sitong Wu, Yinghao Zhu, et~al.
\newblock Scaf-grpo: Scaffolded group relative policy optimization for
  enhancing llm reasoning.
\newblock \emph{arXiv preprint arXiv:2510.19807}, 2025{\natexlab{a}}.

\bibitem[Zhang et~al.(2026)]{zhang2026aero}
Zhi Zhang et~al.
\newblock Train less, learn more: Adaptive efficient rollout optimization for
  group-based reinforcement learning.
\newblock \emph{arXiv preprint}, 2026.

\bibitem[Zhang et~al.(2025{\natexlab{b}})]{zhang2025rlvmr}
Zijing Zhang et~al.
\newblock {RLVMR}: Reinforcement learning with verifiable meta-reasoning
  rewards for robust long-horizon agents.
\newblock \emph{arXiv preprint}, 2025{\natexlab{b}}.

\bibitem[Zhou et~al.(2024)Zhou, Zanette, Pan, Levine, and
  Kumar]{zhou2024archer}
Yifei Zhou, Andrea Zanette, Jiayi Pan, Sergey Levine, and Aviral Kumar.
\newblock {ArCHer}: Training language model agents via hierarchical multi-turn
  {RL}.
\newblock \emph{arXiv preprint arXiv:2402.19446}, 2024.

\bibitem[Zhu et~al.(2025)]{zhu2025stratified}
Mingkang Zhu et~al.
\newblock Stratified {GRPO}: Handling structural heterogeneity in reinforcement
  learning of {LLM} search agents.
\newblock \emph{arXiv preprint}, 2025.

\end{thebibliography}
\bibliographystyle{iclr2026_conference}

\appendix
\section{Registered Predictions Ledger}
\label{app:prereg}

Kept verbatim with outcomes.
Each preregistered prediction was committed as a full 64-character SHA256 digest of
its plaintext to a version-controlled file \emph{before} the corresponding run
completed; the digest file and all plaintexts ship with the code release, so every
``registered before outcome'' claim below is independently checkable.

{\small
\begin{longtable}{p{5.2cm}p{2.2cm}p{5.2cm}}
\caption{Registered predictions and outcomes (excerpt: the most informative
entries; the full 25-entry hash registry, which also covers the estimator,
WebShop, ScienceWorld, and 8B-family arms, ships with the code
release).}\label{tab:prereg}\\
\toprule
registered prediction & outcome & note \\
\midrule
\endfirsthead
\toprule
registered prediction & outcome & note \\
\midrule
\endhead
Dose--response: weaker collapses earlier, 8B latest & falsified $\to$ revised & replaced by
saturation-race law (val-zero: 1.7B@30 $<$ 8B@65 $<$ 4B@85) \\
1.7B collapse at ${\sim}$step 30 & hit & \\
Anneal arm: no recovery after $\lambda\to 0$ & hit & irreversibility evidence
\#2; the turn precedes the anneal's tail ($\lambda{\approx}0.068$ at the turn) \\
Rescue arm: no collapse without std & hit & causal localization \\
Arm D: upper-bound $\Phi$ hard to learn & hit & base-rate plateau 0.90--0.93 \\
Arm D: prediction converts to success & not testable at training caliber; revised
to partial hit by formal eval & prediction did not exceed base rate in training;
the 140-game endpoint shows partial behavioral conversion (24.3\% vs.\ 10.7\%
no-signal) and the probe $+0.16$ internalized knowledge (App.~D) \\
Self-report: confidence saturates day one & hit & variance-trajectory retro instance \\
Self-report: harmless under std (variance ${\to}0$) & hit at training caliber;
revised by formal eval & training last-6 49.5\% $=$ baseline; the 140-game endpoint
revises the arm to 38.6\%, below baseline, the table's largest downward revision
(\S\ref{sec:e01}) \\
Anchor-QA: harmless, no gain (revised pre-launch) & premise falsified; formal
endpoint level with baseline & recall did not saturate (0.65${\to}$0.95); the
formal endpoint (44.3\%) lands level with the 44.1\% baseline mean, and drag
appears only at training caliber \\
$\Delta$acc progress reward: safe under std & premise falsified; outcome consistent
with trajectory form & accuracy plateaued at 0.92, variance persisted, chronic drag
(28.1\%); endpoint form falsified prospectively a second time \\
Group-size ablation ($n{=}8$, constant budget): collapse delayed but not prevented & hit
& literal zero at step 135 vs.\ 85; all-fail fraction climbs from an early
${\sim}0.5$--$0.6$ (rolling) to a terminal lock at $1.00$, measuring the
feedback loop directly \\
Group-size ablation ($n{=}16$): collapse further delayed & half-hit & collapsed
(\checkmark) at step 80, \emph{earlier} than $n{=}4$; the two-point delay trend was
falsified by the third point \\
$\lambda$-sweep ($\lambda \in \{0.01,0.03,0.1\}$): std runs behave identically & hit
(mechanism) & advantage scales $\pm 7.3/\pm 9.1/\pm 7.2$, no trend over tenfold
$\lambda$; all collapse; timing within variance band \\
All-fail-group filter: rescues without touching std & hit & 0\% $\to$ 39.6\%; partial
recovery ($-9.9$ vs.\ baseline), consistent with discarding 23--77\% of samples \\
Aux-loss gain replicates across seeds & hit at two of three seeds & matched-seed
pairing $+19.8/+24.4/-1.6$ at seeds 0/42/96 (vs.\ the three-seed baseline mean:
$+22.6/+16.8/+3.3$); magnitude varies, and the third seed lands at its own baseline \\
Random-vocabulary strong placebo: gains like gold/shuffled $\Rightarrow$ pure
update/format effect; baseline band $\Rightarrow$ environment statistics are the
active ingredient & hit (first branch) & 73.4\%, statistically tied with gold
(69.3) and shuffled (76.0); attribution resolved to the update/format effect \\
R46a (compute-matched replay, no aux target): no gain & hit & 48.9\% $\approx$ baseline; compute/update count contributes nothing \\
R46b (random-token placebo, format-stripped): between baseline and random-vocabulary placebo & hit &
61.5\%, dead-center of the preregistered band; graded attribution ladder complete \\
Rescue from inside the absorbing state (mean-only switch): no recovery & hit & 11
checkpoints at literal zero over a 50-step window; advantage scale confirms amplifier
removed; the absorbing state is self-sustaining \\
Fast anneal completed inside the honeymoon window: avoids collapse & hit &
40.1\%; timed withdrawal is the schedule-side preventive route \\
Matched-variance noise reward: collapses, honeymoon absent & half-hit & no honeymoon
(\checkmark), severe drag $-26$pt (\checkmark), but no absorbing state: collapse
requires variance $\times$ hackability; criterion refined to two dimensions \\
\bottomrule
\end{longtable}}

\section{Additional Figures}
\begin{figure}[H]
  \centering
  \includegraphics[width=\linewidth]{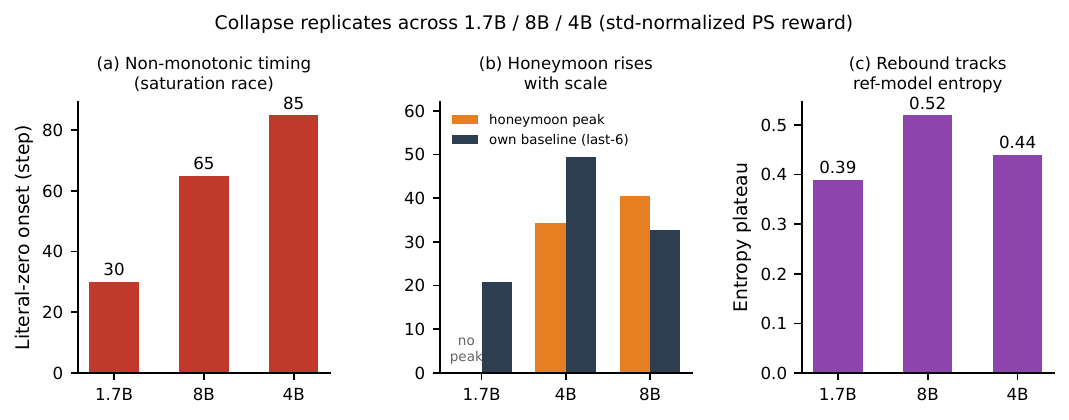}
  \caption{Three-scale summary: (a) non-monotonic literal-zero onset; (b) honeymoon
  peaks rising with scale (baseline bars are training-caliber last-6; the 8B formal
  baseline is 45.0, \S\ref{sec:e01}, which no peak exceeds); (c) post-collapse
  entropy plateaus track reference-model entropy.}
\end{figure}

\section{Reproducibility Details}
\label{app:config}
\begin{table}[H]
\centering\small
\begin{tabular}{l p{7.9cm}}
\toprule
train batch $\times$ group size & 8 $\times$ 4 ($=$ 32 trajectories/step) \\
validation & 32 episodes (ALFWorld/WebShop), 16 (HiddenRule-Gym at 1.7B; the 4B HRG capacity arm used 32); sampled decoding, temperature 0.4, every 5 steps \\
horizon / max steps & 50 (ALFWorld), 30 (HiddenRule-Gym), 15 (WebShop) \\
optimizer / lr & AdamW, $10^{-6}$, KL loss coef 0.01 (low-var estimator) \\
lengths & prompt 1536 / response 512 tokens \\
task reward & terminal; ALFWorld/HiddenRule-Gym binary success, WebShop
continuous matching score in $[0,1]$ \\
shaping & $\lambda = 0.1$, difference-form (superficially PBRS), injected per step-sample \\
advantage normalization unit & GRPO mean/std pooled over all step-samples of a group
(cross-step pooling, the framework default; pool $=$ 4 trajectories $\times$
$\le$50 steps), z-score broadcast to each sample's response tokens (\S\ref{sec:props}) \\
invalid-action penalty & 0.1 \\
auxiliary CE weight (loss-channel arms) & $\beta = 0.1$ (full weight; the 8B
half-weight arm uses $\beta = 0.05$) \\
training length & 150 steps ($\approx$ 18--24 h per run; the post-hoc rescue
extension runs from the step-125 checkpoint to step 175) \\
precision & FP32 master weights; BF16 compute (FSDP mixed precision, FP32 gradient
reduce); BF16 rollout (vLLM). Critic (PPO/GAE arm) follows the actor scheme \\
hardware & 8$\times$96\,GB GPUs per arm (4B/8B); 2$\times$32\,GB (1.7B, HRG) \\
\bottomrule
\end{tabular}
\caption{Shared configuration across all comparison arms (seed 0).}
\end{table}

\section{Published reference points on these benchmarks}
\label{app:sota}

For context only: recent RL-trained agent results on the same benchmarks, with
numbers quoted verbatim from the cited papers. None is comparable to our arms.
They use different base models (Qwen2.5-Instruct family), full training budgets
(ours run at one quarter budget by design, \S\ref{sec:exp-setup}), and their own
evaluation protocols and splits; we include them to locate the setup in the
literature, not to rank against it. All of our claims remain same-budget
arm-to-arm comparisons.

\begin{table}[H]
\centering\scriptsize
\begin{tabular}{p{2.6cm} p{2.2cm} p{4.6cm} p{2.9cm}}
\toprule
method & model & reported result & benchmark(s) \\
\midrule
GRPO (as reported in \citealp{feng2025gigpo,he2026hgpo}) & Qwen2.5-1.5B / 7B &
${\sim}73$--$74\%$ / ${\sim}79$--$83\%$ & ALFWorld \\
GiGPO \citep{feng2025gigpo} & Qwen2.5-1.5B--7B & ${>}{+}12$ / ${>}{+}9$ over
GRPO & ALFWorld / WebShop \\
HCAPO \citep{tan2026hcapo} & Qwen2.5-7B & $+13.8$ / $+7.7$ over GRPO &
ALFWorld / WebShop \\
SCPO \citep{xu2026scpo} & Qwen2.5-1.5B & $93.7 \pm 4.1\%$ / $74.8 \pm 2.0\%$ &
ALFWorld / WebShop \\
BEACON \citep{wang2026beacon} & --- & $92.9\%$, long-horizon split (their
GRPO: $53.5\%$) & ALFWorld; also WebShop, ScienceWorld \\
RLVMR \citep{zhang2025rlvmr} & 7B & $83.6\%$, hardest unseen split &
ALFWorld, ScienceWorld \\
\bottomrule
\end{tabular}
\caption{Published reference points (quoted from the cited papers' reported
results; protocols, budgets, and splits differ from ours and from each other).}
\label{tab:sota}
\end{table}

\section{Additional Experimental Detail}
\label{app:detail}

\paragraph{Boundary conditions of the rescue.}
Std normalization is genuinely useful where it was designed: an adaptive gradient
with provable convergence benefits when within-group reward variance is
task-informative \citep{ge2026curvature}. Our results delimit the failure regime
(a dense policy-dependent shaping term whose spread dominates the within-group
variance; sparse terminal success with all-fail groups is the extreme case), and
independent evidence that std removal repairs GRPO in stochastic-outcome domains
\citep{bereket2025uncalibrated} suggests the repair generalizes along the axis
of ``variance not aligned with task quality.'' A learned critic derives its
scale from bootstrapped value estimates rather than the spread of a small
same-prompt group, so the all-fail degeneracy behind
Proposition~\ref{prop:lambda} has no direct analogue there; the completed
PPO/GAE arm confirms it (identical recipe, 150 steps, no collapse window,
37.5\%, prediction accuracy peaking at 0.60), and the RLOO pair replicates the
pattern inside the group-relative family (40.6\% vs.\ its own 44.8\% baseline).
The danger localizes to the within-group $\sigma$ division, not group-relative
estimation as such (\S\ref{sec:generalization}).

\paragraph{Post-hoc belief probes (HRG).}
Privileged-latent probes on final checkpoints (leakage-audited questions; yes/no
cells reported against majority-class baselines) give the belief-side view.
State-memory probe accuracy orders the arms
$C{=}0.48\ (0.61) > \text{task-weighted } (0.46) > \text{pure GRPO } (0.40) >
C{=}0.90 \approx \text{Arm D } (0.35) > \text{generic-PS } (0.31)$ against a
0.14 random floor: the training-time F1 finding (generic saturated signal erodes
beliefs, task-relevant weighting preserves them) is reproduced by an independent
instrument. Every arm except the mid-coverage and task-weighted ones sits at or
below the no-signal 0.40, including the privileged-feature arm D (0.35), not
only the two pathological regimes (saturation 0.31, overload 0.35): a
non-covering, overloaded, or narrowly targeted $\Phi$ erodes general state
memory, while mid-coverage and task-weighted $\Phi$ strengthen it (0.61, 0.46).
Arm D's low general-memory score coexists with its
$+0.16$ on the specific latent it was trained on (below): the cost is paid in
breadth, not on the trained target. Arm D adds a properly-baselined
dissociation: on the latent it was trained to predict it probes at 0.99 against
a 0.83 majority-class baseline (a real $+0.16$ of internalized knowledge; this
post-hoc yes/no probe is a different instrument from the training-time
generative prediction accuracy, whose 0.90--0.93 plateau in \S\ref{sec:hrg} is
scored against the environment, hence the different baselines), while
task success converts only partially (24.3\% vs.\ 10.7\% at the 140-game
endpoint, \S\ref{sec:e01}; the 16-game window had read it as floor). Knowledge
internalization and behavioral conversion remain separable in degree, not
absolutely. No arm exceeds the majority baseline on rule-relevance probes,
consistent with none cracking the rules.

\section{Complete experiment inventory}
\label{app:inventory}

All 74 training arms of the study, grouped by environment. Design-factor columns: $\sigma$ = group-std normalization in the advantage; R = dense prediction signal through the \emph{reward} channel; L = the same signal through the \emph{loss} channel (auxiliary CE); $\varnothing$ = content-free variant (shuffled / random / noise). ``last-6'' is the training-time endpoint; ``seen''/``unseen'' are the formal endpoint evaluations of \S\ref{sec:e01} where available (ALFWorld: 140-game seen and 134-game unseen padded to 140 episodes; 200 goals for WebShop; 140 probe-family games for HiddenRule-Gym). All completed arms ran the full 150-step budget under the aligned configuration (Appendix~\ref{app:config}), with one exception: the post-hoc rescue arm extends its parent run beyond the budget, a 50-step mean-only window from the step-125 checkpoint to step 175. Every substantive arm carried a hash-registered prediction (25 registry
entries, several covering multi-arm batches, including the estimator, WebShop,
ScienceWorld, and 8B-family arms; the registry file ships with the code
release). Appendix~\ref{app:prereg} reproduces the most informative subset with
outcomes. Arms marked (running)/(queued) belong to the frozen final queue. Within
$\sigma$\checkmark R\checkmark\ rows, outcome differences trace to the modifier
named in the Configuration column (group filtering, anneal schedule, decoupled or
mean-only normalization): the $\sigma$ column marks presence of the within-group
divisor, not absence of every mitigation.

\scriptsize
\begin{longtable}{p{3.6cm}ccccc ccc l}
\caption{ALFWorld, Qwen3-4B (main battery)}\label{tab:inv4b}\\
\toprule
 & & \multicolumn{3}{c}{signal} & & \multicolumn{3}{c}{Metrics (\%)} & \\
\cmidrule(lr){3-5}\cmidrule(lr){7-9}
Configuration & $\sigma$ & R & L & $\varnothing$ & seed & last-6$^{(\uparrow)}$ & seen$^{(\uparrow)}$ & unseen$^{(\uparrow)}$ & outcome \\
\midrule
\endfirsthead
\toprule
 & & \multicolumn{3}{c}{signal} & & \multicolumn{3}{c}{Metrics (\%)} & \\
\cmidrule(lr){3-5}\cmidrule(lr){7-9}
Configuration & $\sigma$ & R & L & $\varnothing$ & seed & last-6$^{(\uparrow)}$ & seen$^{(\uparrow)}$ & unseen$^{(\uparrow)}$ & outcome \\
\midrule
\endhead
\multicolumn{10}{l}{\emph{Baselines and controls}}\\
GRPO baseline & \checkmark &  &  &  & 0 & 49.5 & 42.9 & 42.1 & baseline \\
GRPO baseline & \checkmark &  &  &  & 42 & 39.1 & 40.7 & 37.1 & baseline \\
GRPO baseline & \checkmark &  &  &  & 96 & 51.6 & 48.6 & 46.4 & baseline \\
mean-only control &  &  &  &  & 0 & 52.6 & 57.9 & 44.3 & control \\
2-epoch replay (compute) & \checkmark &  &  &  & 0 & 48.9 & 48.6 & 39.3 & parity \\
\midrule\multicolumn{10}{l}{\emph{Reward channel, std normalization on}}\\
PS prediction reward & \checkmark & \checkmark &  &  & 0 & 1.0 & 0.7 & 0.0 & \textbf{collapse} \\
PS prediction reward & \checkmark & \checkmark &  &  & 42 & 0.0 & 0.0 & 0.0 & \textbf{collapse} \\
PS, $\lambda{=}0.01$ & \checkmark & \checkmark &  &  & 0 & 1.0 & — & — & \textbf{collapse} \\
PS, $\lambda{=}0.03$ & \checkmark & \checkmark &  &  & 0 & 0.0 & — & — & \textbf{collapse} \\
PS, group $n{=}8$ & \checkmark & \checkmark &  &  & 0 & 2.6 & — & — & \textbf{collapse} \\
PS, group $n{=}16$ & \checkmark & \checkmark &  &  & 0 & 0.0 & — & — & \textbf{collapse} \\
PS, slow anneal $\lambda{\to}0$ & \checkmark & \checkmark &  &  & 0 & 0.0 & — & — & \textbf{collapse} \\
PS + fast anneal & \checkmark & \checkmark &  &  & 0 & 40.1 & — & — & \textbf{rescue} \\
PS, GiGPO (step groups) & \checkmark & \checkmark &  &  & 0 & 0.0 & — & — & \textbf{collapse} \\
matched-variance noise & \checkmark & \checkmark &  & \checkmark & 0 & 23.4 & — & — & severe drag \\
$\Delta$acc progress reward & \checkmark & \checkmark &  &  & 0 & 28.1 & 32.1 & 29.3 & chronic drag \\
$\Delta$acc progress reward & \checkmark & \checkmark &  &  & 42 & 34.4 & — & — & chronic drag \\
anchor-QA reward & \checkmark & \checkmark &  &  & 0 & 44.3 & 44.3 & 33.6 & training-only drag \\
self-report reward & \checkmark & \checkmark &  &  & 0 & 49.5 & 38.6 & 39.3 & parity/drag \\
\midrule\multicolumn{10}{l}{\emph{Reward channel, normalization modified}}\\
PS + mean-only &  & \checkmark &  &  & 0 & 51.6 & 52.9 & 40.7 & \textbf{rescue} \\
PS + all-fail filter & \checkmark & \checkmark &  &  & 0 & 39.6 & 40.7 & 35.7 & partial rescue \\
PS + decoupled norm & \checkmark & \checkmark &  &  & 0 & 31.2 & 27.9 & 30.0 & drag \\
PS + decoupled norm & \checkmark & \checkmark &  &  & 42 & 42.2 & 32.1 & 26.4 & drag \\
mean-only@step125 (post hoc) &  & \checkmark &  &  & 42 & 0.0 & — & — & no recovery \\
\midrule\multicolumn{10}{l}{\emph{Loss channel (auxiliary CE)}}\\
aux gold & \checkmark &  & \checkmark &  & 0 & 69.3 & 68.6 & 65.0 & gain \\
aux gold & \checkmark &  & \checkmark &  & 42 & 63.5 & 57.9 & 45.0 & gain \\
aux gold & \checkmark &  & \checkmark &  & 96 & 50.0 & 52.1 & 39.3 & parity \\
aux shuffled & \checkmark &  & \checkmark & \checkmark & 0 & 76.0 & 81.4 & 73.6 & gain \\
aux shuffled & \checkmark &  & \checkmark & \checkmark & 42 & 60.4 & 67.9 & 66.4 & gain \\
aux shuffled & \checkmark &  & \checkmark & \checkmark & 96 & 48.9 & — & — & parity \\
aux random-vocab & \checkmark &  & \checkmark & \checkmark & 0 & 73.4 & \textbf{87.1} & 71.4 & gain \\
aux random-vocab & \checkmark &  & \checkmark & \checkmark & 42 & 77.6 & — & — & gain \\
aux random-vocab & \checkmark &  & \checkmark & \checkmark & 96 & 60.9 & — & — & gain \\
aux random-tokens & \checkmark &  & \checkmark & \checkmark & 0 & 61.5 & 67.9 & 57.9 & gain \\
aux noise (R60a) & \checkmark &  & \checkmark & \checkmark & 0 & 43.2 & — & — & no gain \\
aux half-gold (R60b) & \checkmark &  & \checkmark & \checkmark & 0 & (running) & — & — &  \\
\midrule\multicolumn{10}{l}{\emph{Estimator family (same PS recipe)}}\\
PPO/GAE + PS &  & \checkmark &  &  & 0 & 37.5 & — & — & survives \\
RLOO baseline &  &  &  &  & 0 & 44.8 & — & — & baseline \\
RLOO + PS &  & \checkmark &  &  & 0 & 40.6 & — & — & survives \\
R++ baseline (E11a) &  &  &  &  & 0 & 55.2 & — & — & baseline \\
R++ + PS (E11b) &  & \checkmark &  &  & 0 & 48.4 & — & — & no collapse (post-freeze) \\
\bottomrule
\end{longtable}

\begin{longtable}{p{3.6cm}ccccc ccc l}
\caption{ALFWorld, Qwen3-8B (bistability family)}\label{tab:inv8b}\\
\toprule
 & & \multicolumn{3}{c}{signal} & & \multicolumn{3}{c}{Metrics (\%)} & \\
\cmidrule(lr){3-5}\cmidrule(lr){7-9}
Configuration & $\sigma$ & R & L & $\varnothing$ & seed & last-6$^{(\uparrow)}$ & seen$^{(\uparrow)}$ & unseen$^{(\uparrow)}$ & outcome \\
\midrule
\endfirsthead
\toprule
 & & \multicolumn{3}{c}{signal} & & \multicolumn{3}{c}{Metrics (\%)} & \\
\cmidrule(lr){3-5}\cmidrule(lr){7-9}
Configuration & $\sigma$ & R & L & $\varnothing$ & seed & last-6$^{(\uparrow)}$ & seen$^{(\uparrow)}$ & unseen$^{(\uparrow)}$ & outcome \\
\midrule
\endhead
baseline & \checkmark &  &  &  & 0 & 32.8 & 45.0 & 39.3 & baseline \\
PS prediction reward & \checkmark & \checkmark &  &  & 0 & 0.0 & 0.0 & 0.0 & \textbf{collapse} \\
\midrule
aux gold & \checkmark &  & \checkmark &  & 0 & 0.0 & 0.0 & 0.0 & \textbf{absorbing} \\
aux gold & \checkmark &  & \checkmark &  & 42 & 72.4 & \textbf{85.0} & 72.1 & \textbf{best gain} \\
aux gold & \checkmark &  & \checkmark &  & 96 & 0.0 & — & — & \textbf{absorbing} \\
aux gold, $\beta{=}0.05$ & \checkmark &  & \checkmark &  & 0 & 65.1 & — & — & escape$+$gain \\
aux shuffled & \checkmark &  & \checkmark & \checkmark & 0 & 78.6 & 79.3 & 66.4 & gain \\
prompt-only & \checkmark &  &  &  & 0 & 47.4 & 52.1 & 44.3 & level \\
\bottomrule
\end{longtable}

\begin{longtable}{p{3.6cm}ccccc ccc l}
\caption{ALFWorld, Qwen3-1.7B / Qwen2.5-1.5B}\label{tab:inv17}\\
\toprule
 & & \multicolumn{3}{c}{signal} & & \multicolumn{3}{c}{Metrics (\%)} & \\
\cmidrule(lr){3-5}\cmidrule(lr){7-9}
Configuration & $\sigma$ & R & L & $\varnothing$ & seed & last-6$^{(\uparrow)}$ & seen$^{(\uparrow)}$ & unseen$^{(\uparrow)}$ & outcome \\
\midrule
\endfirsthead
\toprule
 & & \multicolumn{3}{c}{signal} & & \multicolumn{3}{c}{Metrics (\%)} & \\
\cmidrule(lr){3-5}\cmidrule(lr){7-9}
Configuration & $\sigma$ & R & L & $\varnothing$ & seed & last-6$^{(\uparrow)}$ & seen$^{(\uparrow)}$ & unseen$^{(\uparrow)}$ & outcome \\
\midrule
\endhead
1.7B baseline & \checkmark &  &  &  & 0 & 20.9 & 17.9 & 18.6 & baseline \\
1.7B PS reward & \checkmark & \checkmark &  &  & 0 & 0.0 & 0.0 & 0.0 & \textbf{collapse} \\
1.5B full budget (anchor) & \checkmark &  &  &  & 0 & 67.2 & — & — & anchor \\
\bottomrule
\end{longtable}

\begin{longtable}{p{3.6cm}ccccc ccc l}
\caption{WebShop, Qwen3-4B}\label{tab:invws}\\
\toprule
 & & \multicolumn{3}{c}{signal} & & \multicolumn{3}{c}{Metrics (\%)} & \\
\cmidrule(lr){3-5}\cmidrule(lr){7-9}
Configuration & $\sigma$ & R & L & $\varnothing$ & seed & last-6$^{(\uparrow)}$ & seen$^{(\uparrow)}$ & unseen$^{(\uparrow)}$ & outcome \\
\midrule
\endfirsthead
\toprule
 & & \multicolumn{3}{c}{signal} & & \multicolumn{3}{c}{Metrics (\%)} & \\
\cmidrule(lr){3-5}\cmidrule(lr){7-9}
Configuration & $\sigma$ & R & L & $\varnothing$ & seed & last-6$^{(\uparrow)}$ & seen$^{(\uparrow)}$ & unseen$^{(\uparrow)}$ & outcome \\
\midrule
\endhead
base, small & \checkmark &  &  &  & 0 & 60.4 & 60.0 & — & baseline \\
PS reward, small & \checkmark & \checkmark &  &  & 0 & 1.0 & 0.5 & — & \textbf{instant collapse} \\
base, full & \checkmark &  &  &  & 0 & 29.2 & 28.5 & — & baseline \\
PS reward, full & \checkmark & \checkmark &  &  & 0 & 0.5 & 0.5 & — & \textbf{collapse} \\
\midrule
aux gold, small & \checkmark &  & \checkmark &  & 0 & 71.4 & 63.0 & — & gain$\to$tie \\
aux gold, small & \checkmark &  & \checkmark &  & 42 & 50.5 & 46.0 & — & below base \\
aux shuffled, small & \checkmark &  & \checkmark & \checkmark & 0 & 70.3 & \textbf{67.0} & — & gain \\
PS + mean-only, small &  & \checkmark &  &  & 0 & 59.4 & 57.5 & — & \textbf{rescue} \\
\bottomrule
\end{longtable}

\begin{longtable}{p{3.6cm}ccccc ccc l}
\caption{HiddenRule-Gym, Qwen3-1.7B/4B}\label{tab:invhrg}\\
\toprule
 & & \multicolumn{3}{c}{signal} & & \multicolumn{3}{c}{Metrics (\%)} & \\
\cmidrule(lr){3-5}\cmidrule(lr){7-9}
Configuration & $\sigma$ & R & L & $\varnothing$ & seed & last-6$^{(\uparrow)}$ & seen$^{(\uparrow)}$ & unseen$^{(\uparrow)}$ & outcome \\
\midrule
\endfirsthead
\toprule
 & & \multicolumn{3}{c}{signal} & & \multicolumn{3}{c}{Metrics (\%)} & \\
\cmidrule(lr){3-5}\cmidrule(lr){7-9}
Configuration & $\sigma$ & R & L & $\varnothing$ & seed & last-6$^{(\uparrow)}$ & seen$^{(\uparrow)}$ & unseen$^{(\uparrow)}$ & outcome \\
\midrule
\endhead
pure GRPO, 1.7B & \checkmark &  &  &  & 0 & 8.3 & 10.7 & — & floor \\
PS generic $\Phi$ & \checkmark & \checkmark &  &  & 0 & 10.4 & 14.3 & — & floor; F1 collapse \\
PS task-weighted $\Phi$ & \checkmark & \checkmark &  &  & 0 & 5.2 & 8.6 & — & floor; F1 held \\
arm D, privileged $\Phi$ &  & \checkmark &  &  & 0 & 10.4 & 24.3 & — & partial conversion \\
\midrule
coverage $C{=}0.23$ &  & \checkmark &  &  & 0 & 12.5 & 26.4 & — & gain \\
coverage $C{=}0.48$ &  & \checkmark &  &  & 0 & 24.0 & \textbf{34.3} & — & best HRG \\
coverage $C{=}0.66$ &  & \checkmark &  &  & 0 & 14.6 & 13.6 & — & mid dip \\
coverage $C{=}0.86$ &  & \checkmark &  &  & 0 & 22.9 & 22.1 & — & gain \\
coverage $C{=}0.90$ &  & \checkmark &  &  & 0 & 3.1 & 10.7 & — & benefit consumed \\
\midrule
aux gold @$C{=}0.48$ & \checkmark &  & \checkmark &  & 0 & 14.6 & 17.1 & — & marginal \\
aux shuffled @$C{=}0.48$ & \checkmark &  & \checkmark & \checkmark & 0 & 14.6 & 15.0 & — & tied w/ gold \\
pure GRPO, 4B & \checkmark &  &  &  & 0 & 26.6 & — & — & off floor \\
\bottomrule
\end{longtable}

\begin{longtable}{p{3.6cm}ccccc ccc l}
\caption{ScienceWorld, Qwen3-4B}\label{tab:invsw}\\
\toprule
 & & \multicolumn{3}{c}{signal} & & \multicolumn{3}{c}{Metrics (\%)} & \\
\cmidrule(lr){3-5}\cmidrule(lr){7-9}
Configuration & $\sigma$ & R & L & $\varnothing$ & seed & last-6$^{(\uparrow)}$ & seen$^{(\uparrow)}$ & unseen$^{(\uparrow)}$ & outcome \\
\midrule
\endfirsthead
\toprule
 & & \multicolumn{3}{c}{signal} & & \multicolumn{3}{c}{Metrics (\%)} & \\
\cmidrule(lr){3-5}\cmidrule(lr){7-9}
Configuration & $\sigma$ & R & L & $\varnothing$ & seed & last-6$^{(\uparrow)}$ & seen$^{(\uparrow)}$ & unseen$^{(\uparrow)}$ & outcome \\
\midrule
\endhead
base (boil) & \checkmark &  &  &  & 0 & (running) & — & — &  \\
PS reward (boil) & \checkmark & \checkmark &  &  & 0 & (queued) & — & — &  \\
\bottomrule
\end{longtable}

\normalsize

\section{Extended related work}
\label{app:related-ext}
\paragraph{Group-normalized RL and its advantage pathologies.}
Group-relative policy optimization \citep{shao2024grpo,guo2025deepseekr1}
normalizes each group's returns by their mean and standard deviation. Dr.GRPO
\citep{liu2025drgrpo} identified this estimator's std and length biases; we adopt
its mean-only variant as an \emph{instrument} and contribute the causal
demonstration that in long-horizon sparse-success agents the std term alone
separates catastrophic from benign. \citet{bereket2025uncalibrated} independently
show that removing group-std normalization fixes GRPO-induced overconfidence for
stochastic outcomes, while \citet{ge2026curvature} defend std normalization as an
adaptive gradient on reasoning benchmarks; our results delimit where that defense
breaks: phases where an $\varepsilon$-scale shaping term supplies the dominant
within-group variance, with sparse-success all-fail groups as the extreme case.
Related diagnoses include heterogeneous-group bias
\citep{zhu2025stratified}, local-normalization bias \citep{hu2025reinforcepp},
baseline-induced entropy hazards \citep{wu2025qae}, mean-only versus
mean-plus-variance calibration theory \citep{mroueh2025grpo}, and GDPO's decoupled
multi-reward normalization \citep{liu2026gdpo}, which our decoupled arm tests and
finds insufficient under sustained dense pressure. A distinct line treats
\emph{degenerate} (all-same-outcome) groups as zero-signal dead zones to be
pruned, replaced, or reweighted \citep{xu2025spo,zhang2026aero,yang2026biased,zhang2025scaf}.
Our setting inverts that premise: once a shaping term is injected, all-fail groups
are \emph{full-signal} zones, their $\varepsilon$-scale shaping differences
amplified to full advantage scale (Proposition~\ref{prop:lambda}). Remedies that
merely skip or downweight degenerate groups remove the amplifier's fuel without
removing the amplifier: our filter arm shows this rescues only partially
($-9.9$pt vs.\ baseline, \S\ref{sec:rescue}), and by suppressing the symptom
such remedies can mask the mechanism. Concurrent work formalizes the group-std term as the single
quantity on which GRPO, Dr.GRPO, and DAPO differ \citep{bay2026identity},
consistent with our localization; we add what that term \emph{does} when an
injected signal is the only within-group variance. Qualitative precedents of the
amplification exist (``illusory advantages'' in text-to-image GRPO,
\citealp{prefgrpo2025}; scale-equalizing per-component normalization,
\citealp{mogrpo2025}); Proposition~\ref{prop:lambda} is, to our knowledge, its
first formalization for the sparse-success all-fail-group structure of multi-turn
agents. Our infrastructure follows the step-independent multi-turn line
\citep{feng2025gigpo,zhou2024archer}; successors refine step-level advantage
estimation \citep{he2026hgpo} while retaining the group-std normalizer, and our
results concern that shared normalization machinery, not any particular grouping
scheme.

\paragraph{Dense signals, shaping, and reward hacking in agent RL.}
Prediction-\emph{error} seeking \citep{pathak2017curiosity} is famously
vulnerable to the noisy-TV trap \citep{burda2019curiosity,mavorparker2022ama};
our failure is the mirror image, prediction-\emph{accuracy} seeking parking in
predictable corners \citep{friston2012darkroom,baltieri2019darkroom}, a pathology
one deep active-inference agent already exhibited by mastering a single repeated
action \citep{champion2024deconstructing}. On the shaping side, Ng-style
invariance \citep{ng1999policy} assumes a policy-independent potential; dynamic
and learned potentials break it in known ways
\citep{devlin2012dynamic,harutyunyan2015dpba,behboudian2021pies,forbes2024pbim},
but that line never considers the interaction with group normalization, where our
damage concentrates. Dense process rewards are known to be hacking-prone
\citep{cui2025prime,pan2022misspecification,wang2026hackingsurvey}; the standard
mitigation is the progress principle, reward change rather than level
\citep{setlur2024pav,hou2025lpm,wang2025igpo,xi2025agentprm}. Our $\Delta$acc arm
is its preregistered test, and its failure (accuracy plateaus below saturation,
variance persists, chronic drag) sharpens the boundedness principle of
\citet{fu2025par}: return-level boundedness does not survive group z-scoring at
the step level. RAGEN's ``Echo Trap'' \citep{wang2025ragen} and stability reports
for dense turn-level rewards \citep{wang2025practitioner} are near neighbors
observed without a controlled dense-signal culprit.

\paragraph{Auxiliary objectives and world modeling for agents.}
From UNREAL \citep{jaderberg2017unreal} to modern LLM agents, world-model signals
delivered through the loss channel have a consistent track record
\citep{shrivastava2026echo,lu2026paw,voelcker2024selfprediction}. Nor is the
reward channel uniformly hostile: RWML \citep{yu2026rwml} succeeds with an
embedding-space gap reward on ALFWorld while warning that token-level next-state
prediction ``can lead to model collapse,'' a one-line anticipation of the failure
we characterize, and VAGEN \citep{wang2025vagen} rewards world-model reasoning
under a redesigned estimator; our variance-trajectory criterion must be, and
is, consistent with these successes (\S\ref{sec:variance-profile}). What this literature
lacks is a controlled comparison holding the signal fixed while varying only the
consumption mechanism; our matrix supplies it (\S\ref{sec:channel}), with a
placebo ladder, a compute-matched control, and a gradient-interference probe.
Relative to ECHO \citep{shrivastava2026echo}, which established the auxiliary-CE
recipe's effectiveness, we contribute the attribution: neither content--context
pairing nor environment statistics nor extra compute is the active ingredient.

\paragraph{Training instability, seed variance, and scale.}
Seed-dependent instability of language-model fine-tuning is well documented,
including runs that diverge partway through training \citep{dodge2020finetuning},
with vanishing gradients as one driver \citep{mosbach2021stability}. Our 8B
bistability (\S\ref{sec:generalization}) is not this phenomenon rebranded: two
of three gold seeds enter the same early collapse window (the third locks later,
from the lead), the failed basin is a mechanistically
characterized absorbing state (prediction-format saturation under task-gradient
starvation), and the escaped basin lands at the study's best gold-signal result,
a pattern consistent with two distinct attractors rather than diffuse variance,
though with six occupancy samples it remains a structural hypothesis
(\S\ref{sec:limitations}). Scale-dependent instabilities invisible at smaller
scales, and their predictability from internal signals, are documented for
pretraining \citep{wortsman2023smallscale}; our entropy-based precursor plays the
analogous role for agent RL. Finally, the auxiliary-update framing connects to
negative transfer in auxiliary-task learning \citep{jiang2023forkmerge}: at 4B
the auxiliary CE pass is pure positive transfer, at 8B its sign becomes
seed-dependent, a scale boundary auxiliary-task weighting schemes may need to
negotiate.

\section{Collapse mechanism: measurement detail}
\label{app:mech-detail}
\paragraph{Budget vs.\ published ALFWorld numbers.}
Published full-budget GRPO baselines on this stack reach ${\sim}73$--$74\%$ (1.5B) and
${\sim}79$--$83\%$ (7B) on ALFWorld \citep[reported as the GRPO baseline
in][]{feng2025gigpo,he2026hgpo}. Our arms deliberately run at one quarter of that
budget (32 trajectories/step, 150 steps) so that fourteen mutually comparable arms
fit a fixed compute envelope; every claim is a same-budget arm-to-arm comparison,
none depends on absolute success rates, and a full-budget sanity anchor
(Qwen2.5-1.5B, 128 trajectories/step) reaches 67.2\% in our hands, on the
published baseline's convergence trajectory. Published reference points on
these benchmarks, including current state-of-the-art claims, are collected in
Appendix~\ref{app:sota} with the protocol caveats that make them
non-comparable to our same-budget arms.

A replay of a collapsed 1.7B checkpoint against a healthy same-scale policy (12
episodes each, identical harness) confirms the drift directly: state-visitation
entropy drops tenfold (0.91${\to}$0.09; 1.25 locations per episode vs.\ 4.6),
consecutive-action repetition rises 0.21${\to}$0.93, episode length pins at the
horizon, and the collapsed policy is \emph{more} compliant than the healthy one
(invalid-action rate 0.005 vs.\ 0.173, prediction-parse validity 1.00): better behaved on every
surface metric we log, and it solves nothing. Three direct measurements complete
the picture. First, preregistered constant-budget group-size ablations
($4{\times}8$ and $2{\times}16$) test whether the pathology is an artifact of
unstable small-group $\hat{\sigma}$ at $n{=}4$: collapse occurs at every group
size with \emph{non-monotone} timing (literal zero at steps 85/135/80), the
$n{=}8$ delay a two-point comparison would read as a trend being overturned by
the third point; timing is dominated by dynamics variance, not $\hat{\sigma}$
stability. Second, these runs log the all-fail group fraction: it climbs and
locks at $1.00$ (every group all-fail, task gradient identically zero, updates
driven \emph{entirely} by the prediction reward), the lock coinciding with the
literal-zero window in both measured runs (Figure~\ref{fig:groupsize-runaway});
a calibrated retrodiction of the original collapse run shows the same ramp
(${\sim}0.65 \to {\ge}0.9$), while the rescued and baseline arms' fractions
instead decay to ${\sim}0.25$ as success accumulates. Third, the coefficient is
empirically irrelevant inside all-fail groups, as Proposition~\ref{prop:lambda}
predicts: a tenfold reduction leaves the amplified scale unchanged ($\pm 7.2$ at
$\lambda{=}0.1$ vs.\ $\pm 7.3$ at $0.01$) and the arm collapses on schedule
(last-6 1.0\%); annealing has nothing to act on.

\paragraph{Early-warning signature.}
The joint precursor is entropy monotone decline $\wedge$ prediction saturation
$\wedge$ length pinning, leading the turn by 15--30 steps. Entropy alone
false-alarms three ways (victory sharpening: the rescued arm dips to 0.095 while
winning; early sharpening: the decoupled arm dips to 0.094 with prediction
\emph{declining}; pre-breakthrough consolidation: the auxiliary-loss arm reaches
0.04 right before its 28${\to}$59\% phase transition), and it also declines for
reward-independent reasons under clipped policy-gradient training
\citep{park2025cliplow}, so the conjunction is required. We froze the rule as
entropy-has-declined $\wedge$ prediction ${\ge}0.95$ $\wedge$ length
${\ge}0.95\,H$, evaluated over history since the components arrive sequentially.
Its thresholds were calibrated only on the four definition-set collapse runs
(the three-scale collapse arms and the slow-anneal arm); the conjunction
requirement itself was motivated both by those retrodictions and by the
entropy-only dips already observed in healthy runs. Evaluated once, at a fixed
date, on the fourteen runs then available outside the calibration
set,\footnote{Three collapses: group-size $n{=}8$ ($+73$), the
$\lambda{=}0.01$ arm ($+53$), and group-size $n{=}16$ ($+13$). Eleven training-healthy arms: the
seed-0/42 baselines, mean-only PS and its no-signal control, anchor-QA,
$\Delta$acc, the all-fail filter, the gold, shuffled-gold, and
random-vocabulary auxiliary arms, and self-report. ``Healthy'' means no
sustained zero-success validation window at training caliber, so drag arms
count as healthy. Two of the healthy arms
(mean-only PS, aux gold) supplied the pre-freeze entropy-dip observations that
motivated the conjunction, so their clearing is confirmatory rather than
independent; excluding them leaves nine independent healthy arms with the same
single alarm. The decoupled arm, though completed by the evaluation date, is
absent from the recorded roster, an omission we flag rather than justify; at
its entropy dip its prediction accuracy was declining, below the rule's 0.95
gate, so the conjunction would not have alarmed on it, and its inclusion could
only have enlarged the healthy denominator. Arms completing after the
evaluation date, whether collapsing ($\lambda{=}0.03$, GiGPO, the 4B PS
seed-42 arm, both WebShop arms) or healthy (matched-variance noise, fast
anneal), were never scored against the rule; three-of-three is a claim over
the evaluated set, not over all later collapses.}
it detects all three
held-out \emph{reward-channel} collapses, 73/53/13 steps ahead of the collapse
point, defined as the first validation opening a sustained dead band
(${\le}6.5\%$ with no recovery; the first \emph{exactly}-zero checkpoint can
lag it, e.g.\ step 135 vs.\ 130 for $n{=}8$, and the 15--30-step precursor
lead above is measured to the still-earlier \emph{turn}: three distinct
calibers). Its scope is
that failure class: the 8B auxiliary-loss collapses (two locked seeds; the
direct response-length measurement is from seed 0) bypass the length-pinning
component altogether (\S\ref{sec:generalization}), structurally outside the
rule's coverage rather than a miss. Among the eleven arms healthy at training
caliber it raises one alarm, on the self-report arm, whose formal endpoint later
revises below baseline (38.6\%, \S\ref{sec:e01}); that cuts both ways (a true
detection the 32-game window could not resolve, or a false positive on an arm
that merely underperformed; the arm never collapses, so the variance criterion's
no-collapse call stands), and we report the record conservatively: three of
three detected, at most one false alarm.

\section{Fix hierarchy and signal-form detail}
\label{app:fix-detail}
\paragraph{Irreversibility.}
Constant pressure: 120 post-collapse steps at fixed $\lambda$ (1.7B), zero recovery.
Annealing: cosine $\lambda\to 0$ still collapses, the turn arriving at
$\lambda{\approx}0.068$ (step ${\sim}57$ of the schedule; literal zero follows
by ${\sim}85$), and its
final ${\sim}30$ steps at $\lambda{<}0.01$ (effectively pure GRPO) show zero
recovery; entropy remelts to 0.439, matching the un-annealed arm's 0.44
(Fig.~\ref{fig:collapse-fixes}). Post-collapse entropy rebounds at all scales
(1.7B 0.39, 4B 0.44, 8B 0.52; ordered by reference-model entropy, consistent with KL-anchored
dynamics; \citealp{cui2025entropy}): the absorbing state is \emph{behavioral},
not temperature-level, and restoring entropy does not restore the policy, so
fixes must be preventive. The
preregistered rescue arm completes the intervention set: resuming the seed-42
collapsed arm from its step-125 checkpoint, 70 steps past that arm's first
zero-success evaluation, and switching to mean-only normalization (the
intervention that prevents collapse from the start) recovers \emph{nothing} over
a 50-step window, 11 validation checkpoints at literal 0\%. The measured
advantage scale confirms the amplifier is gone: magnitudes fall from
$\pm 5$--$9$ to ${\approx}0.1$, the raw per-step scale that $\lambda$ and the
invalid-action penalty share. Prediction accuracy sits saturated at $1.0$, so
the shaping term contributes almost no within-group spread; the residual
${\approx}0.1$ excursions are raw-scale outliers (consistent with the collapsed
policy's rare invalid steps, rate $0.005$), not amplified signal, and there is
no task gradient to escape with: the absorbing state is self-sustaining. Late
scheduling, coefficient scale, and post-hoc structural intervention all fail;
three preventive routes work: structural (mean-only from the start),
sample-level (dropping all-fail groups, partial rescue at 39.6\%,
\S\ref{sec:rescue}), and \emph{timed withdrawal} (a preregistered fast-anneal
arm completing the same cosine schedule inside the honeymoon window avoids
collapse entirely, 40.1\%, Fig.~\ref{fig:collapse-fixes}). Irreversibility binds
after the turn, not before.

\paragraph{Matched-variance noise isolates hackability.}
A preregistered control replaces the prediction score with random potentials of
matched per-step variance: the amplifier stays fed, but the policy has no control
over the signal. The arm shows no honeymoon, drags severely (23.4\%, $-26$ vs.\
baseline), and yet never enters an absorbing state across 150 steps. Sustained
variance is therefore sufficient for damage but not for the lock; the absorbing
state additionally requires the signal to be \emph{optimizable by the policy}.
This is the second axis of the criterion (\S\ref{sec:variance-profile}); the two
axes together are necessary in every observed reward-channel lock but not jointly sufficient: the
$\Delta$acc and anchor-QA arms satisfy both and neither locks (chronic drag for
$\Delta$acc; training-caliber-only drag for anchor-QA, whose formal endpoint is
level with baseline, \S\ref{sec:e01}). It is
why the noise arm's registered prediction scored a half-hit (Appendix~A).

\paragraph{Decoupling is not enough.}
Per-channel advantage decoupling (GDPO-style; prediction channel independently
normalized, contribution capped at $\lambda$) does not collapse; the cap holds
empirically (0.066--0.077). But it ends at \textbf{31.2\%} last-6, well below the
mean-normalized arms; a second seed lands level with its matched baseline at
training caliber (42.2 vs.\ 39.1) but reproduces the deficit formally (27.9/32.1
against matched baselines 42.9/40.7; descriptive comparison). The measured scales
carry the lesson (Fig.~\ref{fig:dose}): the healthy mean-only arm's raw shaping
advantage ($\pm 0.098$) is \emph{larger} than the dragging decoupled arm's capped
one ($0.066$--$0.077$), so magnitude alone does not order the outcomes; amplification structure
does. The per-channel z-score re-stretches vanishing raw differences to a fixed
$\lambda$-scale pressure regardless of the signal's actual spread, a capped copy
of the same amplifier, whereas the mean-only contribution is the raw
$\lambda(r-\bar r)$, tied to that spread. The ladder reads: unamplified
${\approx}\lambda$ harmless / amplified-but-capped chronic drag / amplified
full-scale lethal, ordered by structure, not magnitude.

\section{Attribution ladder detail}
\label{app:attr-detail}
Two further controls complete a graded attribution ladder. The preregistered
strong placebo (random-vocabulary: environment-disjoint words; SHA256 digest
published before completion) reached \textbf{73.4\%}, statistically tied with
true gold (69.3\%) and shuffled gold (76.0\%); true gold is, if anything, the
lowest of the three. A compute-matched replay control (double PPO epochs, no
auxiliary target) gains nothing (48.9\%), and a format-stripped placebo (CE on
bare out-of-domain words, no schema) captures a majority of the effect (61.5\%).
The ladder reads: baseline 49.5 ${\approx}$ compute replay 48.9 $<$ random-token
placebo 61.5 $<$ schema placebos 73.4--76.0 ${\approx}$ true gold 69.3. The
active ingredient is the auxiliary CE pass on \emph{novel target strings}
(${\sim}+12$ training caliber; $+25$ formal over the matched seed-0 baseline),
with schema-consistent formatting adding $+7.8$ to $+14.5$ at training caliber
that the formal ladder revises to anywhere from $+0.7$ (gold) to $+19.2$
(random-vocabulary) over the format-only rung: the decomposition's magnitudes
are caliber-dependent and not cleanly additive, while the ordering (compute
${\approx}$ baseline $<$ format-only $\leq$ schema placebos) and the
content-free conclusion are stable across calibers. Compute contributes nothing
and content contributes nothing \emph{positive} at any level (formally, gold sits
below the placebo level at both matched seeds: 68.6 vs.\ the three-placebo s0
mean 78.8, and 57.9 vs.\ the single completed s42 placebo endpoint 67.9,
\S\ref{sec:e01}). This is a regularization /
format-anchoring effect, echoing in-context results where label-randomized
demonstrations retain most of their benefit \citep{min2022rethinking}. A
preregistered gradient-noise falsifier closes the remaining alternative:
replacing the auxiliary CE gradient with structure-matched, zero-mean per-token
sign noise (Rademacher $\pm\beta$, everything else identical) lands at 43.2\%
last-6, inside the baseline band and 26 points below gold; arbitrary update
perturbation of matched magnitude buys nothing, so the gain requires
\emph{directed} novel-target gradients, not update noise (training caliber;
registered in the review batch, outside the fourteen-arm matrix). Stated at
its strongest: on class means, a format-consistent auxiliary CE pass on
\emph{arbitrary} strings outperforms every reward-channel delivery of a
genuinely informative signal (the weakest placebo seed, shuffled s96 at 48.9,
lands level with the mean-only arms).
This bears awkwardly on the auxiliary-world-model literature: gains attributed
to world modeling may be substantially attributable to the auxiliary update
itself, a hypothesis our controls isolate but only compute-matched ablations in
those systems can settle.

\section{WebShop and HiddenRule-Gym detail}
\label{app:env-detail}
The reward-channel failure generalizes, and the task regime sets its \emph{form}.
Full-mode WebShop collapses after a parity phase (turn
${\sim}60$, sustained zero from ${\sim}105$, within the 30--135 val-zero range
the ALFWorld arms span once the group-size ablations are included): the
ALFWorld arc with the honeymoon replaced by baseline tracking, so not a
honeymoon in the Table~\ref{tab:collapse} sense. Small-mode WebShop shows
\emph{instant suppression}, completing a four-form taxonomy:
honeymoon-collapse (ALFWorld), collapse-after-parity (full-mode WebShop, the
same delayed arc with the honeymoon replaced by parity), instant suppression,
and chronic drag (the capped arms). The regime split fits the
two-condition account (\S\ref{sec:variance-profile}): in the easy regime
(baseline 60\%) a continuous task score over short episodes lets prediction
variance dominate the group advantage from step~0, while in the hard regime
(baseline 29\%) scores run low and within-group return spread is small,
restoring the runway in which the shaping term comes to dominate, and with it
the delayed
collapse-after-parity arc (early dense rewards biasing learning away from
exploration is a known developmental-RL effect, \citealp{park2025s2d}). The
environment moderates form and timing; occurrence tracks the channel. The
attribution and the rescue both transfer: a shuffled-gold placebo lands at
70.3\%, statistically tied with true gold's 71.4\% (formally 67.0 vs.\ 63.0,
the placebo again at or above gold), so the content-free result
is not an ALFWorld artifact, and mean-only normalization applied to the
instant-suppression regime lands at 59.4\%, parity with the 60.4\% baseline and
a $+58$-point rescue over the std arm. (WebShop has no mean-only no-signal
control, so the parity reading is a cross-normalization comparison like the
HRG references; since the normalizer difference, $+3.1$ training / $+13.8$
formal on ALFWorld, could only flatter the arm, and it still does not beat the
baseline, the no-gain regularity holds a fortiori.) A second auxiliary-loss seed qualifies
the gain sharply: 50.5\% training caliber, $-9.9$ below its baseline, 46.0 vs.\
60.0 formal ($-14.0$); the sign itself flips with the seed, and the first
seed's formal gain compresses to $+3.0$ (63.0 vs.\ 60.0), statistically level.
What replicates across environments is the seed-to-seed bandwidth, not the
gain's direction: the WebShop auxiliary result is a single-seed-pair difference
($+11.0$/$-9.9$ training; $+3.0$/$-14.0$ formal), not a confirmed effect size.

\paragraph{The loss-channel gain is not unconditional.}
On HiddenRule-Gym at 1.7B (coverage-0.48 protocol), the same auxiliary recipe reaches
17.1\% under the 140-game endpoint protocol (\S\ref{sec:e01}): statistically
inseparable from the no-signal GRPO arm (10.7\%; a clean std-to-std comparison)
and below the
coverage-matched reward arm (34.3\%, a 17.2-point gap with disjoint intervals,
under the categorical bar of \S\ref{sec:e01} and cross-normalization: the
reward arm is mean-only while the auxiliary arm keeps std). The gain has conditions
(capacity, coverage structure, environment class) that one environment success
plus a seed-split second (WebShop, \S\ref{sec:generalization}) do
not license extrapolating away. Its shuffled pair lands at a statistically tied
15.0\%: even this boundary-regime non-gain is content-free, so the attribution of
\S\ref{sec:channel} now holds in all three environments.

\paragraph{Does the failure depend on the advantage estimator?}
Four completed estimator arms (GiGPO; PPO/GAE; the two arms of an RLOO
baseline--signal pair) align on a single design bit. GiGPO \citep{feng2025gigpo}, which adds
\emph{step-level} grouping on top of episode-level groups while keeping the
group-std normalizer, collapses on schedule (a healthy 19--34\% band for
${\sim}45$ steps, then sustained literal zero with the familiar terminal
signature: prediction accuracy 1.000, prediction-parse validity 1.000); finer grouping
widens the supply of degenerate groups rather than fixing them. PPO/GAE (a
learned critic, no group statistics) survives the identical recipe: 150 steps,
no collapse window, prediction accuracy peaking at 0.60, a 37.5\% endpoint. RLOO
(a leave-one-out mean baseline: grouped, but no std division) survives likewise
at 40.6\% against its own baseline arm's 44.8\%. Collapse therefore tracks
whether the advantage \emph{divides by a within-group $\sigma$}, not grouping
granularity or critic versus group baseline; the signal (present in every cell
except the RLOO baseline) is necessary for the failure
(\S\ref{sec:rescue}) but does not by itself decide it. Two
qualifiers keep this honest: all estimator endpoints here are training-caliber
(formal evaluations pending at the freeze), so the load-bearing contrast is
collapse versus none, not the orderings (the PPO arm moreover has no matched
PPO baseline, and the RLOO pair's $-4.2$ sits inside single-seed noise). The
REINFORCE++ pair (global-batch whitening) was the last cell: its baseline arm
completed healthy (55.2\% last-6, the strongest baseline in the estimator
family), and its signal arm, registered before launch and completing after the
main data freeze, resolves the cell on the survival side: 48.4\% last-6 with
no collapse window at any point (series minimum 15.6\%, no dead band), no
terminal triple at step 150 (prediction accuracy 0.556, mean episode length
33.7 of 50, entropy 0.185), and an all-fail-group fraction of 0.375 rather
than a lock at 1.0. Had it collapsed, the warning would have widened to all
std normalization; its survival localizes the warning to the within-group
$\sigma$ division (training-caliber, single-seed, post-freeze; its $-6.8$
against the pair baseline sits inside the single-seed noise band and is not
interpreted).

\subsection{Separating coverage, dynamics, and capacity}

HRG's exactly computable coverage lets us separate what ALFWorld conflates.

\paragraph{Coverage axis.}
Arm A (pure GRPO): statistically at the random floor (last-6 8.3\%, formal 10.7\%,
against 7.7\%, the empirical success rate of a
uniform-random valid-action policy over 1{,}000 episodes) with \emph{gradient
starvation} (grad norm 0.05), the near-zero-difference pathology (residual
within-group spread reduces to occasional invalid-action penalties, and most
groups have none at all, placing this arm in the $\sigma_p \to 0$,
$\epsilon$-floor branch of Proposition~\ref{prop:variance}) and the
opposite pole from $\varepsilon$-amplification.
Arm B (PS, generic $\Phi$): prediction saturates 0.99 in ${\sim}20$ steps, success stays
statistically at the no-signal level (formal 14.3 vs.\ 10.7, overlapping
intervals), and the belief-probe F1 \emph{collapses} 0.51${\to}$0.24. The reward does
control which beliefs are maintained, but here it was pointed at non-covering
content. (No absorbing state forms in this environment: with the baseline itself at
the random floor, there is no success rate to destroy, so the absorbing-state
phenotype cannot manifest; the shaping pressure surfaces instead as the belief-probe
collapse above. This is why the abstract's collapse count excludes
the floor-bound HRG arms.)
Arm C (task-relevant reweighting within the same $\Phi$ family): F1 holds at 0.49, success
still floor; reallocation within a non-covering family does not buy coverage.
Arm D (upper-bound privileged feature, clean mean-only dynamics): prediction plateaus
at the base-rate ceiling 0.90--0.93 without ever cracking the rule; the 16-game
training window read its success as floor (10.4\%), but the 140-game endpoint
evaluation (\S\ref{sec:e01}) revises this to 24.3\% against the 10.7\% no-signal
reference (disjoint intervals, though the 13.6-point gap sits under the
categorical bar of \S\ref{sec:e01}, with the further caveat that the reference arm is
std-normalized while arm D is mean-only, folding in a normalization difference
that measures $+3.1$ training / $+13.8$ formal on ALFWorld and is therefore
material at this gap size, Fig.~\ref{fig:csweep}):
the internalized feature \emph{does} convert
partially, on a scale the small validation window could not resolve. Feature
difficulty remains measurable from the saturation \emph{form}: instant 0.99 means
trivially predictable; a base-rate plateau means beyond capacity.

\paragraph{Capacity axis.}
The same pure-GRPO recipe at 4B lifts off the floor ($26.6\%$ vs.\ the 1.7B arm's
$8.3\%$, ${\approx}3.2\times$; training caliber, 32- vs.\ 16-episode windows;
grad norm ${\sim}1.0$ vs.\ starvation 0.05 at 1.7B) yet plateaus for
all 150 steps: capacity explains the 1.7B floor (necessary) but does not by itself crack
hidden rules (not sufficient). Train success 0.19--0.34 without validation transfer
indicates memorization rather than rule inference. This arm is outside the
formal-evaluation roster, so the capacity reading is training-caliber only.

Family controls localize the lock. A \emph{shuffled-gold} arm (targets randomly
permuted across steps) enters no collapse window and climbs to \textbf{78.6\%}
(formal 79.3, statistically level with the gold escape seed's 85.0): the
\emph{gain} side of the content-free attribution carries to 8B. The \emph{risk}
side does not. A \emph{prompt-only} control (predict block in the prompt, no
training signal) lands at 47.4\%, healthy throughout, so the prompt change alone
neither collapses nor explains the locked run. Across the first four occupancy
samples (gold s0: 0.0; gold s42: 72.4; shuffled: 78.6; prompt-only: 47.4) only
one had locked, and the lock looked like a low-occupancy accident. Two further
family arms complete the occupancy picture (Fig.~\ref{fig:8b-family}) and
overturn that reading. A half-weight arm ($\beta{=}0.05$, on the locked seed)
escapes and finishes at 65.1\%: halving the coefficient turns the locked
absorbing state into an escapable dip, so the coefficient scale participates
causally in basin entry. (No tension with Proposition~\ref{prop:lambda}: the
coefficient cancelled there is the reward-channel $\lambda$ inside a z-scored
advantage; $\beta$ weights a loss term that never passes through advantage
normalization.) A third gold seed (96) answers bimodal-versus-continuous in the
bimodal direction: it leads the field through step 40 (71.9\%), then collapses
to literal zero at ${\sim}$step 85. Gold full-weight at 8B locks two of three
seeds; every content-free or reduced-weight arm stays healthy. The risk
concentrates in the true-gold signal at full weight. Under a content-blind null
at the observed gold lock rate (two of three), the two healthy arms that carry
the auxiliary CE channel (shuffled, half-weight) have probability
$(1/3)^2 \approx 0.11$; counting the prompt-only arm too, which carries no
auxiliary CE and is arguably not lockable under the null, would give
$0.04$. Suggestive either way, but six samples remain six
(\S\ref{sec:limitations}).

\section{Formal evaluation: out-of-distribution pass and calibration}
\label{app:e01-detail}
(ii)~The content-free attribution is now the study's most consistent formal pattern,
stated within scale: at 4B the shuffled-gold s0 placebo (81.4\%) and
random-vocabulary arm (87.1\%) sit above every true-gold seed (68.6/57.9/52.1),
while the shuffled s42 and random-token endpoints (67.9) do not. Across scales the highest
\emph{gold} endpoint is the 8B escape seed (85.0\%; the table's overall maximum
remains the random-vocabulary placebo), whose two sibling seeds lock at zero; at 8B, placebo-versus-gold ordering is replaced by the lottery structure of
\S\ref{sec:generalization}. The attribution ladder restates formally on matched seed-0 arms as 42.9
(baseline) $\approx$ 48.6 (compute control) $<$ 67.9 (random-token placebo)
$\approx$ 68.6 (gold) $<$ 81.4--87.1 (schema placebos): the format-only rung
already matches gold, and the schema placebos exceed it. One calibration note: the
placebo--gold gaps themselves (12.8--18.5 points) fall below the caption's
categorical threshold (the shuffled placebo's interval moreover overlaps gold's;
the random-vocabulary interval clears gold's, but that row is single-seed and the
gap is under 25 points), so what survives is a descriptive class-level ordering,
stated seed-paired (the three-placebo s0 mean 78.8 vs.\ gold 68.6 at seed 0;
the single completed s42 placebo endpoint 67.9 vs.\ 57.9 at seed 42; the pooled
means 76.1 vs.\ 59.5 average four placebo endpoints, three s0 and one s42,
against the three gold seeds), not a categorical claim;
pointwise, two placebo endpoints (67.9) sit just below the strongest gold seed
(68.6), so even per-seed placebo $\geq$ gold does not hold.

\paragraph{Out of distribution, every headline structure survives and every
magnitude compresses.} The five collapsed arms are literal zero there too (upper
bounds 2.7\%): the dark room is a property of the policy, not the evaluation
set. The channel dichotomy holds (every reward-channel variant at or below the
41.9\% unseen baseline mean; the mean-only pair stays tied, 40.7 vs.\ 44.3), and
the content-free ordering is preserved (placebos 57.9--73.6 above the gold mean
49.8). Four small near-baseline gaps do flip sign against their matched seed-0
references (mean-only rescue $+10.0 \to -1.4$, compute control $+5.7 \to -2.8$,
anchor-QA $+1.4 \to -8.5$, and gold s96 against its matched s96 baseline
$+3.5 \to -7.1$), all with overlapping intervals; the claims here rest on the
large gaps, not on those cells. What compresses is
magnitude: the auxiliary gain over baseline goes from $+15.4$ to $+7.9$, the two
later gold seeds falling into the baseline band; direction survives, size
roughly halves, the seed-variance theme repeats. The largest single reversion is
the random-vocabulary placebo ($87.1 \to 71.4$, ceding the top spot to
shuffled-gold at 73.6): within-class \emph{rankings} are seed-unstable even
though the class-level claim (no placebo seed below the baseline band or
collapsing; training-caliber class means 66.2 vs.\ gold 60.9) is
seed-replicated.

The weakening: self-report drops to 38.6\% (from a training-time 49.5), the
table's largest downward revision, moving it from ``parity'' to ``below the
baseline mean at point estimate'' (the interval still overlaps the baseline's).
A single 32-game-window number moved 11 points under formal evaluation: that is
the justification for this section, and the reason the single-seed rows above,
including the table-topping placebos, carry the caption's cross-seed caveat
until their seed replications complete.

\paragraph{Resolved since the previous version.}
The concern that group size $n{=}4$ is collinear with the std pathology (small
groups give unstable $\hat{\sigma}$) was closed by preregistered constant-budget
ablations at $n{=}8$ and $n{=}16$ (\S\ref{sec:failure}): collapse occurs at
every group size with non-monotone timing (85/135/80), so the pathology is
neither a small-group artifact nor reliably mitigated by larger groups. The
single-environment concern was closed by the WebShop replication
(\S\ref{sec:generalization}).

\section{Extended practical guidance}
\label{app:guidance-ext}

\begin{enumerate}
  \item Do not mix difference-form self-prediction shaping into GRPO's
  std-normalized reward when within-group task-return variance is small relative
  to the shaping term. The sparse-success all-fail regime is the extreme case
  (our collapse configuration), and the WebShop easy regime shows a second route
  (continuous scores over short episodes let prediction variance dominate the
  group advantage from step~0). The prescription is conditional: dense
  reward-channel supervision with different variance trajectories has succeeded
  on this very benchmark (\S\ref{sec:related}). The environment changes the
  failure's \emph{form} while leaving its occurrence intact
  (\S\ref{sec:generalization}).
  \item If reward-channel shaping is unavoidable, use mean-only normalization.
  Per-channel decoupling under-performs (formal 27.9/32.1 across its two seeds
  vs.\ the 52.9--57.9\% mean-only pair, a 21--30-point deficit; $-12$ to $-16$pt
  vs.\ the baseline mean at point estimate), and annealing rescues only when
  timed early: completed \emph{inside} the honeymoon window it prevents collapse
  (40.1\%), while late annealing fails and no post-hoc rescue attempt succeeded
  (irreversibility, \S\ref{sec:satrace}).
  \item Prefer signals whose within-group variance decays over training, but
  verify empirically that they actually saturate in your environment. Our
  preregistered $\Delta$acc test (echoing the progress principle of
  \citealp{setlur2024pav} and \citealp{hou2025lpm}) is the cautionary result:
  accuracy plateaued below saturation, progress variance persisted, and the arm
  ran 12 points below the baseline mean at formal point estimate (32.1\% vs.\
  44.1\%, intervals overlapping; a second seed reproduces the direction at
  training caliber, 34.4 vs.\ its matched 39.1). Constructions that force
  variance to zero at mastery do not bind if mastery is not reached.
  \item If you want the gain, add an auxiliary CE pass \emph{as a regularizer},
  and know what you are buying. Any well-formed target works: under formal
  evaluation at 4B, content-free placebo targets as a class match or beat
  true-gold targets (seed-paired: the three-placebo s0 mean 78.8 vs.\ gold
  68.6; the single completed s42 placebo endpoint 67.9 vs.\ 57.9; the
  single-seed
  random-vocabulary endpoint, 87.1\%, tops the table). \textbf{Do not pay for
  gold process supervision, and do not expect the gain to scale with signal
  quality}: the content is inert for the gain (\S\ref{sec:channel},
  \S\ref{sec:e01}), the gain comes from the update itself and not from update
  noise (a structure-matched sign-noise falsifier stays at the baseline band,
  \S\ref{sec:channel}), with the task reward
  untouched and zero measured gradient interference, and on WebShop the outcome is
  seed-split ($+11.0$/$-9.9$ training caliber; formal $+3.0$ statistically level
  and $-14.0$; hard-regime cell untested). Treat model scale as hazardous: at 8B
  the same recipe is bistable, the risk concentrating in true-gold targets at
  full weight (two of three gold seeds lock; shuffled, prompt-only, and
  half-weight arms all stay healthy; a structural reading from six occupancy
  samples, two of them content-free, \S\ref{sec:limitations}). At scale, run a seed pair before trusting
  it, and consider halving the auxiliary weight; in our single test that turned
  the locked configuration into an escapable one (65.1\%)
  (\S\ref{sec:generalization}).
  \item Monitor the joint precursor (entropy decline $\wedge$ prediction
  saturation $\wedge$ length pinning), never entropy alone (three false-alarm
  modes documented).
  \item Check feature-set difficulty via the prediction-saturation form before
  spending compute: instant 0.99 $=$ trivially predictable feature set;
  base-rate plateau $=$ beyond model capacity.
\end{enumerate}

\paragraph{Seeds and statistics.}
The headline auxiliary-loss arm carries three seeds (69.3/63.5/50.0: direction
stable at seeds 0/42, magnitude $-1.6$ to $+24.4$ under matched-seed pairing),
and the ALFWorld baseline carries three (49.5/39.1/51.6). The WebShop
replication carries two (training 71.4/50.5; formal 63.0/46.0 against a 60.0
baseline), reproducing the seed-to-seed spread while the gain's sign flips
($+11.0$/$-9.9$ training, $+3.0$/$-14.0$ formal), so we report the
cross-environment gain as a single-seed-pair difference, not a confirmed effect
size. Most mechanism arms are single-seed; the decoupled and $\Delta$acc arms
carry a second seed each (decoupled 31.2/42.2 last-6, an 11-point spread that is
itself a measurement of this uncertainty, both formal endpoints below their
matched baselines; $\Delta$acc reproduces the drag direction at training
caliber), so cross-arm gaps inherit seed variance that is only partially
quantified. The 8B family carries six occupancy samples (gold 0.0/72.4/0.0,
shuffled 78.6, prompt-only 47.4, half-weight 65.1); the third gold seed answered
bimodal-versus-continuous in the bimodal direction, and the layered pattern
(gold full-weight locks two of three; every content-free or reduced-weight arm
healthy) points at true gold at full weight as the risk concentration, but six
points cannot estimate occupancy rates: the structural description is not a
probability estimate. The 32-game validation limitation is now mostly closed:
every completed arm's final checkpoint in Table~\ref{tab:e01} has a formal
140-game evaluation and the headline gaps survive it, while the mechanism
battery ($\lambda$ ladder, group sizes, anneal pair, noise, post-hoc rescue) and
the cells the caption marks as pending (estimator controls, 8B gold s96 and
half-weight, the $\Delta$acc second seed, late placebo seeds) keep
training-caliber endpoints only (\S\ref{sec:e01}). The out-of-distribution pass
(134-game unseen split; ALFWorld arms only, WebShop and HiddenRule-Gym have no
unseen split) is complete: every headline structure is preserved and
magnitudes compress (four near-baseline gaps flip sign within noise,
\S\ref{sec:e01}), so distribution-shift claims rest on measured numbers.

\section{Criterion: compatibility and identifiability detail}
\label{app:criterion-detail}

Whether a signal saturates fast enough to exit the danger window is empirical,
and our anchor-QA arm is the cautionary instance: we predicted fast saturation,
recall instead climbed slowly (0.65${\to}$0.95) and variance persisted; the arm
showed mild drag at training caliber, though its formal endpoint lands level
with the baseline mean (44.3 vs.\ 44.1, \S\ref{sec:e01}), so the instance
supports the premise (saturation cannot be assumed) rather than the damage
prediction. The other three signal forms then separate as the criterion predicts.
Difference-form scoring keeps fluctuating per step as long as accuracy is high
but short of saturation and the policy still visits varied states, so amplification
persists through the entire approach to collapse; only inside the absorbing
state, where accuracy pins at $1.0$ and behavior degenerates to a fixed loop,
does the variance reach zero
(\S\ref{sec:satrace}), when no task gradient is left to escape with. Saturated
always-positive confidence becomes a constant: variance ${\to}0$, the
signal-fed amplification ends (what remains amplified is the penalty spread,
which every healthy baseline also carries), no collapse (the self-report arm ran at baseline parity throughout
training; its formal endpoint later revises low, a revision whose reading, true
detection or measurement artifact, is left open in \S\ref{sec:failure}). Progress-style $\Delta$acc rewards have variance
${\to}0$ at mastery \emph{by construction}, but our preregistered test is the
second cautionary instance. The arm's reward is
$r_{\Delta\text{acc}} = \lambda\,\mathrm{clip}(\Phi_t - \Phi_{t-1}, 0, 1)$: the
positive clipping already removes it from the difference-form family of
\S\ref{sec:setting}, since a nonnegative reward's episode sum no longer
telescopes to a bounded term. The running accuracy mean plateaued at 0.92 while
the per-step score $\Phi_t$ kept fluctuating around it, so positive increments
kept occurring: the unscaled clip term still paid a per-step mean of 0.095 at
step 150 ($\lambda \times 0.095$ injected), variance persisted, and the arm
ended at 28.1\% (chronic drag, the same outcome rung as the decoupled arm);
the endpoint form of the criterion would have called this arm safe, the
trajectory form correctly does not. A matched-variance noise control adds the
second axis: sustained variance supplies the pressure, but the absorbing state
also requires the signal to be \emph{hackable} (policy-controllable); noise with
the same variance drags severely yet never locks (\S\ref{sec:rescue}). The two
axes are necessary in every \emph{reward-channel} lock we observe (the 8B
auxiliary-loss lock, \S\ref{sec:generalization}, arises without either and sits
outside this criterion's scope) but not jointly sufficient: the $\Delta$acc and
anchor-QA arms sustain variance on policy-controllable signals and neither
locks. What further separates the locking arms (plausibly, the existence of a
degenerate behavior that maximizes the signal at zero task cost, as a fully
predictable loop does for difference-form prediction) remains uncharacterized.
One more
$\varepsilon$-scale term shares the pool: the invalid-action penalty ($0.1$, the
same order as $\lambda$), which varies per step-sample even inside an all-fail
group. Strictly, the amplifier stretches the \emph{joint} spread of shaping and
penalty, rewarding predictability and surface compliance together, precisely the
pair the collapsed policy ends up maximizing (invalid rate $0.005$,
prediction-parse validity $1.00$, \S\ref{sec:failure}).
Proposition~\ref{prop:lambda} states both limits explicitly, and the measured
three-point sweep shows no $\lambda$ trend, consistent with either. The sweep
alone therefore cannot apportion the pooled spread between the two terms
(penalty-dominated pools are $\lambda$-invariant precisely because the shaping
is negligible there). The decomposition is \emph{implementable} as a
training-time monitor, since the two terms enter the reward separately; our
runs logged only their batch means, so we do not report the share, and the
apportioning we rely on is the causal one below. Every baseline arm carries the same penalty
under the same normalizer and the same early all-fail exposure and never
collapses, so penalty spread alone does not produce the failure; adding the
shaping term does (the knock-out matrix, \S\ref{sec:rescue}), and the terminal
behavior is prediction-optimal (accuracy driven to $1.0$), the imprint of
shaping pressure that dominated during the approach (at saturation itself the
shaping contribution has vanished, Proposition~\ref{prop:variance}).

The criterion is also consistent with the published reward-channel successes
of \S\ref{sec:related}, in the weak sense that neither success contradicts it.
RWML's embedding-space gap reward is level-scored, outside the difference-form
family whose per-step fluctuation feeds the amplifier; for level signals our
own pair shows the outcome hinges on whether within-group spread decays
(self-report) or persists (anchor-QA), a property we cannot measure in RWML
from outside, and its authors' warning that token-level next-state prediction
``can lead to model collapse'' independently marks the corner we map. VAGEN
redesigns the estimator, and our GiGPO arm shows redesign alone does not
decide (the deciding bit is whether the within-group $\sigma$ division
survives, a coordinate VAGEN's description leaves unreported). The criterion
issues no safety certificates here; it is compatible with these successes, not
confirmed by them.

\end{document}